\documentclass{article}

    \PassOptionsToPackage{numbers, compress}{natbib}

    \usepackage[final]{neurips_2021}

\usepackage[utf8]{inputenc} 
\usepackage[T1]{fontenc}    
\usepackage{hyperref}       
\usepackage{url}            
\usepackage{booktabs}       
\usepackage{amsfonts}       
\usepackage{nicefrac}       
\usepackage{microtype}      
\usepackage{xcolor}         
\usepackage{subfigure}
\usepackage{multirow}
\usepackage{xspace}
\usepackage{listings}
\usepackage{xcolor}
\usepackage[linesnumbered,vlined,ruled,noend]{algorithm2e}
\usepackage{wrapfig}
\usepackage{enumitem}
\usepackage{CJKutf8}
\usepackage{floatrow}
\newfloatcommand{capbtabbox}{table}[][\FBwidth]
\usepackage{graphicx}
\usepackage{amsmath,amsfonts,graphicx}
\usepackage{amsthm}
\newtheorem{theorem}{Theorem}[section]

\newtheorem{lemma}[theorem]{Lemma}
\newtheorem{definition}{Definition}[section]

\usepackage{babel}
\usepackage{caption}

\usepackage{makecell}
\usepackage{CJKutf8}

\title{Node Dependent Local Smoothing for Scalable\\ Graph Learning}

\author{%
  Wentao Zhang$^{1}$, Mingyu Yang$^1$, Zeang Sheng$^1$, Yang Li$^1$ \\ \textbf{Wen Ouyang$^2$, Yangyu Tao$^2$, 
  Zhi Yang$^{1,3}$, Bin Cui$^{1,3,4}$} \\
  $^1$School of CS, Peking University $^2$Tencent Inc.\\
  $^3$ Key Lab of High Confidence Software Technologies, Peking University\\
  $^4$Institute of Computational Social Science, Peking University (Qingdao), China\\
  $^1$\{wentao.zhang, ymyu, shengzeang18, liyang.cs, yangzhi, bin.cui\}@pku.edu.cn\\ $^2$\{gdpouyang, brucetao\}@tencent.com
}

\begin{document}

\maketitle

\begin{abstract}
Recent works reveal that feature or label smoothing lies at the core of Graph Neural Networks (GNNs). Concretely, they show feature smoothing combined with simple linear regression achieves comparable performance with the carefully designed GNNs, and a simple MLP model with label smoothing of its prediction can outperform the vanilla GCN. Though an interesting finding, smoothing has not been well understood, especially regarding how to control the extent of smoothness. Intuitively, too small or too large smoothing iterations may cause {\em under-smoothing} or {\em over-smoothing} and can lead to sub-optimal performance. Moreover, the extent of smoothness is node-specific, depending on its degree and local structure. To this end, we propose a novel algorithm called node-dependent local smoothing (NDLS), which aims to control the smoothness of every node by setting a node-specific smoothing iteration. Specifically, NDLS computes influence scores based on the adjacency matrix and selects the iteration number by setting a threshold on the scores. Once selected, the iteration number can be applied to both feature smoothing and label smoothing. Experimental results demonstrate that NDLS enjoys high accuracy -- state-of-the-art performance on node classifications tasks, flexibility -- can be incorporated with any models, scalability and efficiency -- can support large scale graphs with fast training.

\end{abstract}

\section{Introduction}

In recent years, Graph Neural Networks (GNNs) have received a surge
of interest with the state-of-the-art performance on many graph-based
tasks~\cite{cai2020multi,zhang2019attributed, DBLP:journals/chinaf/GuoQX021, DBLP:conf/sigmod/ZhangMSJCR020, DBLP:journals/corr/abs-2011-02260, wu2020garg}.
Recent works have found that the success of GNNs can be mainly attributed to
smoothing, either at feature or label level. For example, SGC~\cite{wu2019simplifying} shows using smoothed features as input to a simple linear regression model achieves comparable performance with lots of carefully designed and complex GNNs.
At the smoothing stage, features of neighbor nodes are aggregated and combined with
the current node's feature to form smoothed features. This process is often iterated
multiple times. The smoothing is based on the assumption that labels of
nodes that are close to each other are highly correlated, therefore, the
features of nodes nearby should help predict the current node's label.

\begin{figure*}[tp!]
\centering  
\subfigure[Two nodes with different local structures]{
\label{fig:ob1}
\scalebox{0.4}{
   \includegraphics[width=1\linewidth]{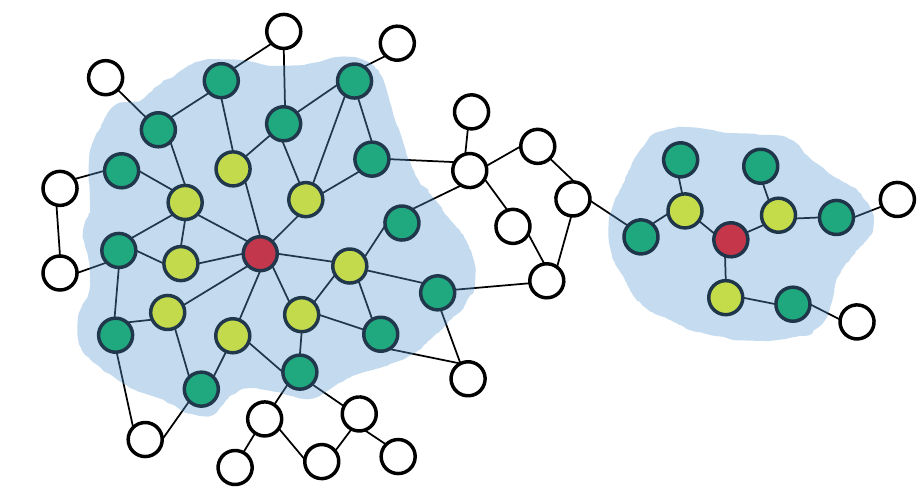}
 }}
\centering
\subfigure[The CDF of LSI in different graphs]{
\label{fig:ob2}
\scalebox{0.5}{
   \includegraphics[width=1\linewidth]{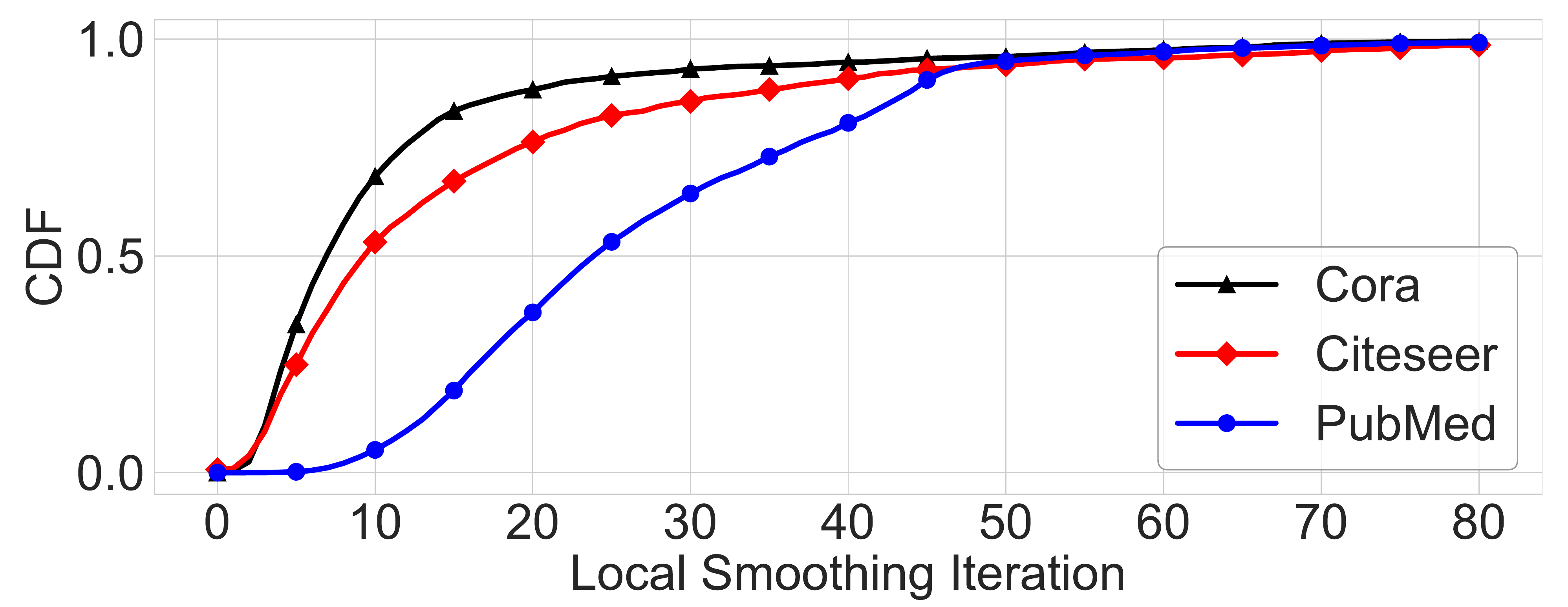}
 }}
\caption{  (Left) The node in dense region has larger smoothed area within two iterations of propagation. (Right) The CDF of LSI in three citation networks.
}
\label{fig.observation}
\vspace{-2mm}
\end{figure*}

One crucial and interesting parameter of neighborhood feature aggregation
is the number of smoothing iterations $k$, which
controls how much information is being gathered. Intuitively, an aggregation
process of $k$ iterations (or layers) enables a node to leverage information
from nodes that are $k$-hop away~\cite{ DBLP:conf/kdd/MiaoGZHLMRRSSWW21,DBLP:conf/kdd/ZhangJLSSMWY021}.
The choice of $k$ is closely related to the structural properties of graphs and has a significant impact on the model performance. However, most existing
GNNs only consider the fixed-length propagation paradigm -- a uniform $k$ for all the nodes. This is problematic since the number of iterations
should be {\em node dependent} based on its degree and local structures.
For example, as shown in Figure~\ref{fig:ob1}, the
two nodes have rather different local structures, with the left red one resides
in the center of a dense cluster and the right red one on the periphery with few connections.
The number of iterations to reach an optimal level of smoothness are rather
different for the two nodes. Ideally, poorly connected nodes (e.g., the red node on the right)
needs large iteration numbers to efficiently gather information from other
nodes while well-connected nodes (e.g., the red node on the left) should keep the iteration number small
to avoid {\em over-smoothing}. 
Though some learning-based approaches have proposed
to adaptively aggregate information for each node through gate/attention
mechanism or reinforcement learning~\cite{spinelli2020adaptive,lai2020policy,DBLP:journals/corr/abs-2108-00955, miao2021lasagne},
the performance gains are at the cost of increased training complexity, hence
not suitable for scalable graph learning.

In this paper, we propose a simple yet effective solution to this problem. Our approach,
called node-dependent local smoothing (NDLS), calculates a node-specific iteration
number for each node, referred to as local smooth iteration (LSI). Once the LSI for a specific node is computed, the corresponding
local smoothing algorithm only aggregates the information from the nodes within
a distance less than its LSI as the new feature. 
The LSI is selected based on influence scores, which measure how other nodes influence
the current node. 
NDLS sets the LSI for a specific node to be the minimum number
of iterations so that the influence score is $\epsilon$-away from the {\em over-smoothing} 
score, defined as the influence score at infinite iteration. 
The insight is that each node's influence score should be at a reasonable level.
Since the nodes with different local structures have different ``smoothing speed'',
we expect the iteration number to
be adaptive. Figure~\ref{fig:ob2} illustrates Cumulative Distribution Function
(CDF) for the LSI of individual nodes in real-world graphs.
The heterogeneous and long-tail property exists in all the datasets,
which resembles the characteristics of the degree distribution of nodes in real graphs.

Based on NDLS, we propose a new graph learning algorithm with three stages:
(1) feature smoothing with NDLS (NDLS-F); (2) model training with smoothed
features; (3) label smoothing with NDLS (NDLS-L). Note that in our framework, 
the graph structure information is only used in pre-processing and post-processing 
steps, i.e., stages (1) and (3) (See Figure~\ref{Fig.overview}). Our NDLS turns a graph learning problem into a
vanilla machine learning problem with independent samples. This simplicity enables
us to train models on larger-scale graphs. Moreover, our NDLS kernel can act as
a drop-in replacement for any other graph kernels and be combined with
existing models such as Multilayer Perceptron (MLP), SGC~\cite{wu2019simplifying}, SIGN~\cite{rossi2020sign}, S$^2$GC~\cite{zhu2021simple} and GBP~\cite{chen2020scalable}.

Extensive evaluations on seven benchmark datasets, including large-scale datasets like ogbn-papers100M~\cite{hu2021ogb}, demonstrates that NDLS achieves not only the state-of-the-art node classification performance but also high training scalability and efficiency. 
Especially, NDLS outperforms APPNP~\cite{spinelli2020adaptive} and GAT~\cite{velivckovic2017graph} by a margin of $1.0\%$-$1.9\%$ and $0.9\%$-$2.4\%$ in terms of test accuracy, while achieving up to $39\times$ and $186\times$ training speedups, respectively.

\section{Preliminaries}
In this section, we first introduce the semi-supervised node classification task
and review the prior models, based on which we derive our method in Section 3.
Consider a graph $\mathcal{G}$ = ($\mathcal{V}$, $\mathcal{E}$) with $|\mathcal{V}| = n$ nodes and $|\mathcal{E}| = m$ edges, the adjacency matrix (including self loops) is denoted as $\tilde{\mathbf{A}} \in \mathbb{R}^{n \times n}$ and the feature matrix is denoted as $\mathbf{X} = \{\boldsymbol{x}_1,\boldsymbol{x}_2 ..., \boldsymbol{x}_n\}$ in which $\boldsymbol{x}_i\in\mathbb{R}^{f}$ represents the feature vector of node $v_{i}$. Besides, $\mathbf{Y} = \{\boldsymbol{y_1},\boldsymbol{y_2} ..., \boldsymbol{y_l}\}$ is the initial label matrix consisting of one-hot label indicator vectors.
The goal is to predict the labels for nodes in the unlabeled set $\mathcal{V}_u$ with the supervision of labeled set $\mathcal{V}_l$.

\textbf{GCN}
smooths the representation of each node via aggregating its own representations and the ones of its neighbors’. This process can be defined as
\begin{equation}
    \mathbf{X}^{(k+1)}=\delta\left(\hat{\mathbf{A}}\mathbf{X}^{(k)}\mathbf{W}^{(k)}\right), \qquad \hat{\mathbf{A}} = \widetilde{\mathbf{D}}^{r-1}\tilde{\mathbf{A}}\widetilde{\mathbf{D}}^{-r},
\label{eq_GCN}
\end{equation}
where $\hat{\mathbf{A}}$ is the normalized adjacency matrix, $r \in [0, 1]$ is the convolution coefficient, and $\widetilde{\mathbf{D}}$ is the diagonal node degree matrix with self loops. 
Here $\mathbf{X}^{(k)}$ and $\mathbf{X}^{(k+1)}$ are the smoothed node features of layer $k$ and $k+1$ respectively while $\mathbf{X}^{(0)}$ is set to $\mathbf{X}$, the original feature matrix.
In addition, $\mathbf{W}^{(k)}$ is a layer-specific trainable weight matrix at layer $k$, and $\delta(\cdot)$ is the activation function.
By setting $r = $ 0.5, 1 and 0, the convolution matrix $\widetilde{\mathbf{D}}^{r-1}\tilde{\mathbf{A}}\widetilde{\mathbf{D}}^{-r}$ represents the symmetric normalization adjacency matrix $\widetilde{\mathbf{D}}^{-1/2}\tilde{\mathbf{A}}\widetilde{\mathbf{D}}^{-1/2}$~\cite{DBLP:conf/iclr/KlicperaBG19}, the transition probability matrix $\tilde{\mathbf{A}}\widetilde{\mathbf{D}}^{-1}$~\cite{DBLP:conf/iclr/ZengZSKP20}, and the reverse transition probability matrix $\widetilde{\mathbf{D}}^{-1}\tilde{\mathbf{A}}$~\cite{xu2018representation}, respectively. 

\textbf{SGC.}
For each GCN layer defined in Eq.~\ref{eq_GCN}, if the non-linear activation function $\delta(\cdot)$ is an identity function and $\mathbf{W}^{(k)}$ is an identity matrix, we get the smoothed feature after $k$-iterations propagation as $\mathbf{X}^{(k)}=\hat{\mathbf{A}}^{k}\mathbf{X}$.
Recent studies have observed that GNNs primarily derive their benefits from performing feature smoothing over graph neighborhoods rather than learning non-linear hierarchies of features as implied by the analogy to CNNs~\cite{liu2020towards,cui2020adaptive,he2020lightgcn}.
By hypothesizing that the non-linear transformations between GCN layers are not critical, SGC~\cite{wu2019simplifying} first extracts the smoothed features $\mathbf{X}^{(k)}$ then feeds them to a linear model, leading to higher scalability and efficiency.
Following the design principle of SGC, piles of works have been proposed to further improve the performance of SGC while maintaining high scalability and efficiency, such as SIGN~\cite{rossi2020sign}, S$^2$GC~\cite{zhu2021simple} and GBP~\cite{chen2020scalable}.

\textbf{Over-Smoothing~\cite{li2018deeper} issue.}
\label{label_smooth}
By continually smoothing the node feature with infinite number of propagation in SGC, the final smoothed feature $\mathbf{X}^{(\infty)}$ is
\begin{equation}
    \begin{aligned}
    &\mathbf{X}^{(\infty)}=\hat{\mathbf{A}}^{\infty}\mathbf{X},  \qquad \hat{\mathbf{A}}^{\infty}_{i,j}  =  \frac{(d_i+1)^r(d_j+1)^{1-r}}{2m+n},
    \label{Over}
    \end{aligned}
\end{equation}
where $\hat{\mathbf{A}}^{\infty}$ is the final smoothed adjacency matrix, $\hat{\mathbf{A}}^{\infty}_{i,j}$  is the weight between nodes $v_i$ and $v_j$, $d_i$ and $d_j$ are the node degrees for $v_i$ and $v_j$, respectively. 
Eq.~\eqref{Over} shows that as we smooth the node feature with an infinite number of propagations in SGC, the final feature is over-smoothed and unable to capture the full graph structure information since it only relates with the node degrees of target nodes and source nodes. For example, if we set $r = 0$ or 1, all nodes will have the same smoothed features because only the degrees of the source or target nodes have been considered.

\section{Local Smoothing Iteration (LSI)}\label{sec:LS}
The features after $k$ iterations of smoothing is $\mathbf{X}^{(k)}=\hat{\mathbf{A}}^{k}\mathbf{X}$.
Inspired by~\cite{xu2018representation}, 
we measure the influence of node $v_{j}$ on node $v_{i}$ by measuring how much a change in the input feature of $v_{j}$ affects the representation of $v_{i}$ after $k$ iterations. For any node $v_{i}$, the influence vector captures
the influences of all other nodes.
Considering the $h^{th}$ feature of $\mathbf{X}$, we define an influence matrix $I_{h}(k)$:
\begin{equation}
I_{h}(k)_{ij} = \frac{\partial \hat{\mathbf{X}}^{(k)}_{ih}}{\partial \hat{\mathbf{X}}^{(0)}_{jh}}.
\end{equation}

\begin{equation}
I(k)=\hat{\mathbf{A}}^{k},\tilde{I}_{i}=\hat{\mathbf{A}}^{\infty}
\end{equation}

Since $I_{h}(k)$ is independent to $h$, we replace $I_{h}(k)$ with $I(k)$, which can be further represented as $I(k) = I_{h}(k),\ \forall h \in \{1,2,..,f\}$, 
where $f$ indicates the number of features of $\mathbf{X}$.
We denote $I(k)_{i}$ as the $i^{th}$ row of $I(k)$, and $\tilde{I}$ as $I(\infty)$. 
Given the normalized adjacency matrix $\hat{\mathbf{A}}$, we can have $I(k)=\hat{\mathbf{A}}^{k}$ and $\tilde{I}=\hat{\mathbf{A}}^{\infty}$. 
According to Eq.~\eqref{Over}, $\tilde{I}$ converges to a unique stationary matrix independent of the distance between nodes, resulting in that the aggregated features of nodes are merely relative with their degrees (i.e., over-smoothing).

We denote $I(k)_{i}$ as the $i^{th}$ row of $I(k)$, and it means the influence from the other nodes to the node $v_{i}$ after $k$ iterations of propagation.
We introduce a new concept {\em local smoothing iteration} (parameterized by $\epsilon$), which measures the minimal number of iterations $k$ required for the influence of other nodes on node $v_{i}$ to be within an $\epsilon$-distance to the over-smoothing stationarity $\tilde{I}_{i}$.

\begin{definition}[] {\bf Local-Smoothing Iteration} (LSI, parameterized by $\epsilon$) is defined as 
\begin{equation}
K(i,\epsilon) = \operatorname{min}\{k: ||\tilde{I}_{i}-I(k)_{i}||_{2}<\epsilon\},
\end{equation}
where $||\cdot||_{2}$ is two-norm, and $\epsilon$ is an arbitrary small constant with $\epsilon > 0$. 
\end{definition}

Here $\epsilon$ is a graph-specific parameter, and a smaller $\epsilon$ indicates a stronger smoothing effect. The $\epsilon$-distance to the over-smoothing stationarity $\tilde{I}_{i}$ ensures that the smooth effect on node $v_{i}$ is sufficient and bounded to avoid over-smoothing.
As shown in Figure~\ref{fig:ob2}, we can have that the distribution of LSI owns the {\em heterogeneous and long-tail property}, where a large percentage of nodes have much smaller LSI than the rest. 
Therefore, the required LSI to approach the stationarity is heterogeneous across nodes.
Now we discuss the connection between LSI and node local structure, showcasing nodes in the sparse region (e.g., both the degrees of itself and its neighborhood are low) can greatly prolong the iteration to approach
over-smoothing stationarity. 
This heterogeneity property is not fully utilized
in the design of current GNNs, leaving
the model design in a dilemma between unnecessary iterations for a majority of nodes and insufficient iterations for the rest of nodes.
Hence, by adaptively choosing the iteration based on LSI for different nodes, we can significantly improve model performance.

\paragraph{Theoretical Properties of LSI.}
We now analyze the factors determining the LSI of a specific node. 
To facilitate the analysis, we set the coefficient $r=0$ for the normalized adjacency matrix $\hat{\mathbf{A}}$ in Eq.~\eqref{eq_GCN}, thus $\hat{\mathbf{A}} = \widetilde{\mathbf{D}}^{-1}\tilde{\mathbf{A}}$. The proofs of following theorems can be found in Appendix A.1.
\begin{theorem}
\label{theorem3.1}
Given feature smoothing $\mathbf{X}^{(k)}=\hat{\mathbf{A}}^{k}\mathbf{X}$ with $\hat{\mathbf{A}} = \widetilde{\mathbf{D}}^{-1}\tilde{\mathbf{A}}$, we have
\begin{equation}
K(i,\epsilon)\le \operatorname{log}_{\lambda_{2}}\left(\epsilon\sqrt{\frac{\tilde{d_{i}}}{2m+n}}\right),
\end{equation}
where $\lambda_{2}$ is the second largest eigenvalue of $\hat{\mathbf{A}}$, $\tilde{d_{i}}$ denotes the degree of node $v_i$ plus 1 (i.e., $\tilde{d_{i}} = d_{i}+1$), 
and $m$, $n$ denote the number of edges and nodes respectively.
\end{theorem}

Note that $\lambda_{2} \le {1}$. 
Theorem \ref{theorem3.1} shows that the upper-bound of the LSI is positively correlated with the scale of the graph ($m,n$), the sparsity of the graph (small $\lambda_{2}$ means strong connection and low sparsity, and vice versa), and negatively correlated with the degree of node $v_i$.

\begin{theorem}
\label{theorem3.2}
For any nodes $i$ in a graph $\mathcal{G}$,
\begin{equation}
K(i,\epsilon)\le \operatorname{max}\left\{K(j,\epsilon),j\in N(i)\right\}+1,
\end{equation}
where $N(i)$ is the set of node $v_{i}$'s neighbours.
\end{theorem}
Theorem \ref{theorem3.2} indicates that the difference between two neighboring nodes' LSIs is no more than $1$, therefore the nodes with a super-node as neighbors (or neighbor's neighbors) may have small LSIs. 
That is to say, the sparsity of the local area, where a node locates, also affects its LSI positively.
Considering Theorems \ref{theorem3.1} and \ref{theorem3.2} together, we can have a union upper-bound of $K(i,\epsilon)$ as
\begin{equation}
K(i,\epsilon)\le \operatorname{min}\left\{\operatorname{max}\left\{K(j,\epsilon),j\in N(i)\right\}+1,\operatorname{log}_{\lambda_{2}}\left(\epsilon\sqrt{\frac{\tilde{d_{i}}}{2m+n}}\right)\right\}.
\end{equation}

\section{NDLS Pipeline}
The basic idea of NDLS is to utilize the LSI heterogeneity to perform a node-dependent aggregation over a neighborhood within a distance less than the specific LSI for each node.
Further, we propose a simple pipeline with three main parts
(See Figure~\ref{Fig.overview}): (1) a node-dependent local smoothing of the feature (NDLS-F) over the graph, (2) a base prediction result with the smoothed feature, (3) a node-dependent local smoothing of the label predictions (NDLS-L) over the graph. 
Note this pipeline is not trained in an end-to-end way, the stages (1) and (3) in NDLS are only the pre-processing and post-processing steps, respectively. 
Furthermore, the graph structure is only used
in the pre/post-processing NDLS steps, not for
the base predictions. 
Compared with prior GNN models, this key design enables higher scalability and a faster training process.

Based on the graph structure, we first compute the node-dependent {\em local smoothing iteration} that maintains a proper distance to the over-smoothing stationarity.
Then the corresponding local smoothing kernel only aggregates the information (feature or prediction) for each node from the nodes within a distance less than its LSI value. 
The combination of NDLS-F and NDLS-L takes advantage of both label smoothing (which tends to perform fairly well on its own without node features) and the node feature smoothing.
We will see that combining these complementary signals yields
state-of-the-art predictive accuracy. Moreover, our NDLS-F kernel can act as
a drop-in replacement for graph kernels in other scalable GNNs such as SGC, S$^2$GC, GBP, etc.

\begin{figure}[tpb]
    \centering
    \includegraphics[width=0.9\textwidth]{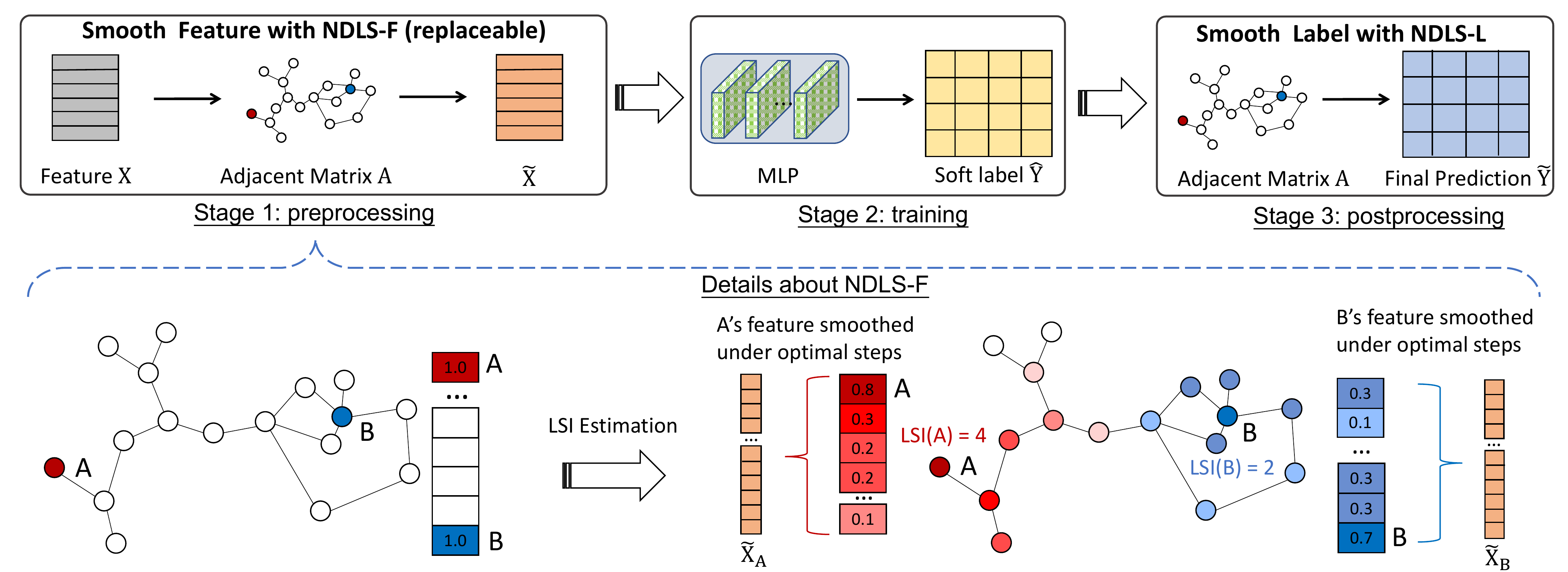}
    \caption{Overview of the proposed NDLS method, including (1) feature smoothing with NDLS (NDLS-F), (2) model training with smoothed
features, and (3) label smoothing with NDLS (NDLS-L). NDLS-F and NDLS-L correspond to pre-processing and post-processing steps respectively.
}  
     \vspace{-2mm}
    \label{Fig.overview}
\end{figure}

\subsection{Smooth Features with NDLS-F}
Once the node-dependent LSI $K(i,\epsilon)$ for a specific node $i$ is obtained, we smooth the initial input feature ${\mathbf{X}}_i$ of node $i$ with node-dependent LSI as:
\begin{equation}
\widetilde{\mathbf{X}}_i(\epsilon) =  \frac{1}{K(i,\epsilon)+1}\sum_{k=0}^{K(i,\epsilon)} {\mathbf{X}}_i^{(k)}.
\label{FLS}
\end{equation}
To capture sufficient neighborhood information, for each node $v_{i}$, we average its multi-scale features $\{\mathbf{X}^{(k)}_i \mid k \leq K(i,\epsilon)\}$ obtained by aggregating information within $k$ hops from the node $v_{i}$.

The matrix form of the above equation can be formulated as
\begin{equation}
\widetilde{\mathbf{X}}(\epsilon) = \sum_{k=0}^{\max\limits_{i} K(i,\epsilon)}
\mathbf{M}^{(k)} {\mathbf{X}}^{(k)}, \qquad \mathbf{M^{(k)}}_{ij}=
\begin{cases}
\frac{1}{K(i,\epsilon)+1}, \quad i=j \quad and \quad k \le K(i,\epsilon)\\
0, \quad \quad \quad \operatorname{otherwise}
\end{cases},
\label{eqm}
\end{equation}
where $\mathbf{M^{(k)}}$ is a set of diagonal matrix.

 \subsection{Simple Base Prediction}
With the smoothed feature $\widetilde{\mathbf{X}}$ according to Eq.~\ref{FLS}, we then train a model to minimize the loss -- $\sum_{v_i\in \mathcal{V}_l} \ell \left(\boldsymbol{y}_i, f(\widetilde{\mathbf{X}}_i)\right)$, where $\widetilde{\mathbf{X}}_i$ denotes the $i^{th}$ row of $\widetilde{\mathbf{X}}$, $\ell$ is the cross-entropy loss function, and $f(\widetilde{\mathbf{X}}_i)$ is the predictive label distribution for node $v_i$. In NDLS, the default $f$ is a MLP model and $\hat{\mathbf{Y}} = f(\widetilde{\mathbf{X}})$ is its soft label predicted (softmax output).
Note that, many other models such as Random Forest~\cite{liaw2002classification} and XGBoost~\cite{chen2016xgboost} could also be used in NDLS (See more results in Appendix A.2).

 \subsection{Smooth Labels with NDLS-L}
Similar to the feature propagation, we can also propagate the soft label $\hat{\mathbf{Y}}$ with $\hat{\mathbf{Y}}^{(k)}=\hat{\mathbf{A}}^{k}\hat{\mathbf{Y}}$.
Considering the influence matrix of softmax label $J_{h}(k)$.
\begin{equation}
J_{h}(k)_{ij} = \frac{\partial \hat{\mathbf{Y}}^{(k)}_{ih}}{\partial \hat{\mathbf{Y}}^{(0)}_{jh}}.
\end{equation}
According to the definition above we have that 
\begin{equation}
J_{h}(k) = I_{h}(k), \forall h \in \{1,2,..,f\}.
\end{equation}
Therefore, local smoothing can be further applied to address over-smoothing in label propagation.
Concretely, we smooth an initial soft label $\hat{\mathbf{Y}}_i$ of node $v_i$ with NDLS as follows
\begin{equation}
\widetilde{\mathbf{Y}}_i(\epsilon) =  \frac{1}{K(i,\epsilon)+1}\sum_{k=0}^{K(i,\epsilon)} \hat{\mathbf{Y}}_i^{(k)}.
\label{LLS}
\end{equation}
Similarly, the matrix form of the above equation can be formulated as
\begin{equation}
\widetilde{\mathbf{Y}}(\epsilon) = \sum_{k=0}^{\max\limits_{i} K(i,\epsilon)}
\mathbf{M}^{(k)} {\hat{\mathbf{Y}}}^{(k)},
\end{equation}
where $\mathbf{M^{(k)}}$ follows the definition in Eq.~\eqref{eqm}.

\section{Comparison with Existing Methods}
\textbf{Decoupled GNNs.} The aggregation and transformation operations in coupled GNNs (i.e., GCN~\cite{kipf2016semi}, GAT~\cite{velivckovic2017graph} and JK-Net~\cite{xu2018representation}) are inherently intertwined in Eq.~\eqref{eq_GCN}, so the propagation iterations $L$ always equals to the transformation iterations $K$. Recently, some decoupled GNNs (e.g., PPNP~\cite{DBLP:conf/iclr/KlicperaBG19}, PPRGo~\cite{bojchevski2020scaling}, APPNP~\cite{DBLP:conf/iclr/KlicperaBG19}, AP-GCN~\cite{spinelli2020adaptive} and DAGNN~\cite{liu2020towards}) argue the entanglement of these two operations limits the propagation depth and representation ability of GNNs, so they first do the transformation and then smooth and propagate the predictive soft label with higher depth in an end-to-end manner. Especially, AP-GCN and DAGNN both use a learning mechanism to learn propagation adaptively.
Unfortunately, all these coupled and decoupled GNNs are hard to scale to large graphs -- {\em scalability issue} since
they need to repeatedly perform an expensive recursive neighborhood expansion in multiple propagations of the features or soft label predicted. NDLS addresses this issue by dividing the training process into multiple stages.

\textbf{Sampling-based GNNs.} An intuitive method to tackle the recursive neighborhood expansion problem is sampling. As a node-wise sampling method, GraphSAGE~\cite{hamilton2017inductive} samples the target nodes as a mini-batch and samples a fixed size set of neighbors for computing. VR-GCN~\cite{DBLP:conf/icml/ChenZS18} analyzes the variance reduction on node-wise sampling, and it can reduce the size of samples with an additional memory cost.
In the layer level, Fast-GCN~\cite{DBLP:conf/iclr/ChenMX18} samples a fixed number of nodes at each layer, and ASGCN~\cite{DBLP:conf/nips/Huang0RH18} proposes the adaptive layer-wise sampling with better variance control.
For the graph-wise sampling, Cluster-GCN~\cite{chiang2019cluster} clusters the nodes and only samples the nodes in the clusters, and GraphSAINT~\cite{DBLP:conf/iclr/ZengZSKP20} directly samples a subgraph for mini-batch training. We don't use sampling in NDLS since the sampling quality highly influences the classification performance.

\textbf{Linear Models.} Following SGC~\cite{wu2019simplifying}, some recent methods remove the non-linearity between each layer in the forward propagation. SIGN~\cite{rossi2020sign} allows using different local graph operators and proposes to concatenate the different iterations of propagated features. S$^2$GC~\cite{zhu2021simple} proposes the simple spectral graph convolution to average the propagated features in different iterations. In addition, GBP~\cite{chen2020scalable} further improves the combination process by weighted averaging, and all nodes in the same layer share the same weight. In this way, GBP considers the smoothness in a layer perspective way. 
Similar to these works, we also use a linear model for higher training scalability. 
The difference lies in that we consider the smoothness from a node-dependent perspective and each node in NDLS has a personalized aggregation iteration with the proposed local smoothing mechanism.

\begin{table*}[tbp]
\caption{Algorithm analysis for existing scalable GNNs.  $n$, $m$, $c$, and $f$ are the number of nodes, edges, classes,  and feature dimensions, respectively. $b$ is the batch size, and $k$ refers to the number of sampled nodes. $L$ corresponds to the number of times we aggregate features, $K$ is the number of layers in MLP classifiers. For the coupled GNNs, we always have $K = L$.} 
\vspace{-2mm}
    \centering
    \resizebox{.95\linewidth}{!}{
    \begin{tabular}{l|c|c|c|c|c}
        \toprule
        \textbf{Type} & \textbf{Method} & \textbf{Preprocessing and postprocessing} & \textbf{Training}   & \textbf{Inference} & \textbf{Memory} \\
        \midrule
        Node-wise sampling & GraphSAGE & - & $\mathcal{O}(k^{L}nf^2)$ & $\mathcal{O}(k^{L}nf^2)$ & $\mathcal{O}(bk^{L}f + Lf^2)$\\
        \hline
        Layer-wise sampling & FastGCN & - & $\mathcal{O}(kLnf^2)$  & $\mathcal{O}(kLnf^2)$ & $\mathcal{O}(bkLf + Lf^2)$\\
        \hline
        Graph-wise sampling & Cluster-GCN & $\mathcal{O}(m)$ & $\mathcal{O}(Lmf+Lnf^2)$  & $\mathcal{O}(Lmf+Lnf^2)$& $\mathcal{O}(bLf + Lf^2)$\\
        \hline
        \multirow{4}{*}{Linear model} & SGC & $\mathcal{O}(Lmf)$ & $\mathcal{O}(nf^2)$  & $\mathcal{O}(nf^2)$& $\mathcal{O}(bf+ f^2)$\\
        & S$^2$GC & $\mathcal{O}(Lmf)$ & $\mathcal{O}(nf^2)$   & $\mathcal{O}(nf^2)$ & $\mathcal{O}(bf+ f^2)$\\
        & SIGN & $\mathcal{O}(Lmf)$ & $\mathcal{O}(Knf^2)$  & $\mathcal{O}(Knf^2)$ & $\mathcal{O}(bLf+ Kf^2)$\\
        & GBP & $\mathcal{O}(Lnf + L\frac{\sqrt{m\lg{n}}}{\varepsilon})$ & $\mathcal{O}(Knf^2)$   & $\mathcal{O}(Knf^2)$ & $\mathcal{O}(bf+ Kf^2)$\\
        \hline
        \multirow{1}{*}{Linear model}
        & NDLS & $\mathcal{O}(Lmf+Lmc)$ & $\mathcal{O}(Knf^2)$   & $\mathcal{O}(Knf^2)$ & $\mathcal{O}(bf+ Kf^2)$\\
        \bottomrule
    \end{tabular}}
    
    \label{algorithm analysis}
\end{table*}

\begin{table*}[t]
\small
\centering
\resizebox{.99\linewidth}{!}{
\caption{Overview of datasets and task types (T/I represents Transductive/Inductive).} \label{Dataset}
\begin{tabular}{ccccccccc}
\toprule
\textbf{Dataset}&\textbf{\#Nodes}& \textbf{\#Features}&\textbf{\#Edges}&\textbf{\#Classes}&\textbf{\#Train/Val/Test}&\textbf{Type}&\textbf{Description}\\
\midrule
Cora& 2,708 & 1,433 &5,429&7& 140/500/1,000 & T&citation network\\
Citeseer& 3,327 & 3,703&4,732&6& 120/500/1,000 & T&citation network\\
Pubmed& 19,717 & 500 &44,338&3& 60/500/1,000 & T&citation network\\
Industry & 1,000,000 & 64 & 1,434,382 & 253 & 5K/10K/30K&T&short-form video network\\
ogbn-papers100M & 111,059,956 & 128 & 1,615,685,872 & 172 & 
1,207K/125K/214K&T&citation network\\
\midrule
Flickr& 89,250 & 500 & 899,756 & 7 &  44K/22K/22K 
& I &image network\\
Reddit& 232,965 & 602 & 11,606,919 & 41 &  
155K/23K/54K
& I&social network \\
\bottomrule
\label{data}
\end{tabular}}
\end{table*}

Table~\ref{algorithm analysis} compares the asymptotic complexity of NDLS with several representative and scalable GNNs.
In the stage of the preprocessing, the time cost of clustering in Cluster-GCN is $\mathcal{O}(m)$ and the time complexity of most linear models is $\mathcal{O}(Lmf)$. 
Besides, NDLS has an extra time cost $\mathcal{O}(Lmc)$ for the postprocessing in label smoothing.
GBP conducts this process approximately with a bound of $\mathcal{O}(Lnf + L\frac{\sqrt{m\lg{n}}}{\varepsilon})$, where $\varepsilon$ is a error threshold. 
Compared with the sampling-based GNNs, the linear models usually have smaller training and inference complexity, i.e., higher efficiency. 
Memory complexity is a crucial factor in large-scale graph learning because it is difficult for memory-intensive algorithms such as GCN and GAT to train large graphs on a single machine.
Compared with SIGN, both GBP and NDLS do not need to store smoothed features in different iterations, and the feature storage complexity can be reduced from $\mathcal{O}(bLf)$ to $\mathcal{O}(bf)$.

\section{Experiments}
In this section, we verify the effectiveness of NDLS on seven real-world graph datasets. 
We aim to answer the following four questions. \textbf{Q1:} Compared with current SOTA GNNs, can NDLS achieve higher predictive accuracy and why?
\textbf{Q2:} Are NDLS-F and NDLS-L better than the current feature and label smoothing mechanisms (e.g., the weighted feature smoothing in GBP and the adaptive label smoothing in DAGNN)? \textbf{Q3:} Can NDLS obtain higher efficiency over the considered GNN models? \textbf{Q4:} How does NDLS perform on sparse graphs (i.e., low label/edge rate, missing features)? 

\subsection{Experimental Setup}

\textbf{Datasets.}
We conduct the experiments on (1) six publicly partitioned datasets, including four citation networks (Citeseer, Cora, PubMed, and ogbn-papers100M) in~\cite{kipf2016semi, hu2021ogb} and two social networks (Flickr and Reddit) in~\cite{DBLP:conf/iclr/ZengZSKP20}, and (2) one short-form video recommendation graph (Industry) from our industrial cooperative enterprise.
The dataset statistics are shown in Table~\ref{data} and more details about these datasets can be found in Appendix A.3.

\textbf{Baselines.}
In the transductive setting, we compare our method with (1) the coupled GNNs: GCN~\cite{kipf2016semi}, GAT~\cite{velivckovic2017graph} and JK-Net~\cite{xu2018representation}; (2) the decoupled GNNs: APPNP~\cite{DBLP:conf/iclr/KlicperaBG19}, AP-GCN~\cite{spinelli2020adaptive}, DAGNN (Gate)~\cite{liu2020towards}, and PPRGo~\cite{bojchevski2020scaling}; 
(3) the linear-model-based GNNs: MLP, SGC~\cite{wu2019simplifying}, SIGN~\cite{rossi2020sign}, S$^2$GC~\cite{zhu2021simple} and GBP~\cite{chen2020scalable}. 
In the inductive setting, the compared baselines are sampling-based GNNs: GraphSAGE~\cite{hamilton2017inductive}, FastGCN~\cite{DBLP:conf/iclr/ChenMX18}, ClusterGCN~\cite{chiang2019cluster} and GraphSAINT~\cite{DBLP:conf/iclr/ZengZSKP20}.
Detailed descriptions of these baselines are provided in Appendix A.4.

\textbf{Implementations.}
To alleviate the influence of randomness, we repeat each method ten times and report the mean performance. 
The hyper-parameters of baselines are tuned by OpenBox~\cite{DBLP:conf/kdd/LiSZCJLJG0Y0021} or set according to the original paper if available. Please refer to Appendix A.5 for more details.

\begin{table}[tpb!]
\caption{Results of transductive settings. OOM means ``out of memory''.}
\vspace{-2mm}
\centering
{
\noindent
\renewcommand{\multirowsetup}{\centering}
\resizebox{0.9\linewidth}{!}{
\begin{tabular}{ccccccccccc}
\toprule
\textbf{Type}&\textbf{Models}&\textbf{Cora}& \textbf{Citeseer}&\textbf{PubMed}&
{\textbf{\makecell{Industry}}}&{\textbf{\makecell{ogbn-papers100M}}}\\
\midrule
\multirowcell{3}{Coupled}&
GCN& 81.8$\pm$0.5 & 70.8$\pm$0.5 &79.3$\pm$0.7&45.9$\pm$0.4&OOM  \\
&GAT& 83.0$\pm$0.7 & 72.5$\pm$0.7 &79.0$\pm$0.3&46.8$\pm$0.7&OOM   \\
&JK-Net& 81.8$\pm$0.5  & 70.7$\pm$0.7 & 78.8$\pm$0.7 & 47.2$\pm$0.3&OOM   \\
\midrule
\multirowcell{5}{Decoupled}&
APPNP& 83.3$\pm$0.5 & 71.8$\pm$0.5 & 80.1$\pm$0.2&46.7$\pm$0.6&OOM \\
&AP-GCN& 83.4$\pm$0.3& 71.3$\pm$0.5& 79.7$\pm$0.3&46.9$\pm$0.7&OOM \\
&PPRGo& 82.4$\pm$0.2& 71.3$\pm$0.5& 80.0$\pm$0.4&46.6$\pm$0.5&OOM \\
&DAGNN (Gate)& 84.4$\pm$0.5& 73.3$\pm$0.6& 80.5$\pm$0.5&47.1$\pm$0.6&OOM \\
&DAGNN (NDLS-L)$^*$&84.4$\pm$0.6& 73.6$\pm$0.7& 80.9$\pm$0.5&47.2$\pm$0.7&OOM \\
\midrule
\multirowcell{5}{Linear}&
MLP & 61.1$\pm$0.6&61.8$\pm$0.8&72.7$\pm$0.6&41.3$\pm$0.8&47.2$\pm$0.3\\
&SGC & 81.0$\pm$0.2 & 71.3$\pm$0.5 & 78.9$\pm$0.5&45.2$\pm$0.3&63.2$\pm$0.2\\
&SIGN& 82.1$\pm$0.3 & 72.4$\pm$0.8 &79.5$\pm$0.5&46.3$\pm$0.5&64.2$\pm$0.2\\
&S$^2$GC& 82.7$\pm$0.3 & 73.0$\pm$0.2 &79.9$\pm$0.3&46.6$\pm$0.6&64.7$\pm$0.3\\
&GBP& 83.9$\pm$0.7 & 72.9$\pm$0.5 &80.6$\pm$0.4&46.9$\pm$0.7&65.2$\pm$0.3\\
\midrule
\multirowcell{4}{Linear}
&NDLS-F+MLP$^*$& 84.1$\pm$0.6 & 73.5$\pm$0.5 &81.1$\pm$0.6&47.5$\pm$0.7&65.3$\pm$0.5\\ 
&
MLP+NDLS-L$^*$&83.9$\pm$0.6&73.1$\pm$0.8&81.1$\pm$0.6&46.9$\pm$0.7&64.6$\pm$0.4\\
&SGC+NDLS-L$^*$&84.2$\pm$0.2&73.4$\pm$0.5&81.1$\pm$0.4&47.1$\pm$0.6&64.9$\pm$0.3\\
&NDLS$^*$&\textbf{84.6$\pm$0.5}&\textbf{73.7$\pm$0.6}&\textbf{81.4$\pm$0.4}&\textbf{47.7$\pm$0.5}&\textbf{65.6$\pm$0.3}\\
\bottomrule
\end{tabular}}}
\label{transductive}
\end{table}

\begin{minipage}[tpb!]{\textwidth}
  \begin{minipage}[b]{0.45\textwidth}
    \centering
    \scalebox{0.2}{
     \includegraphics{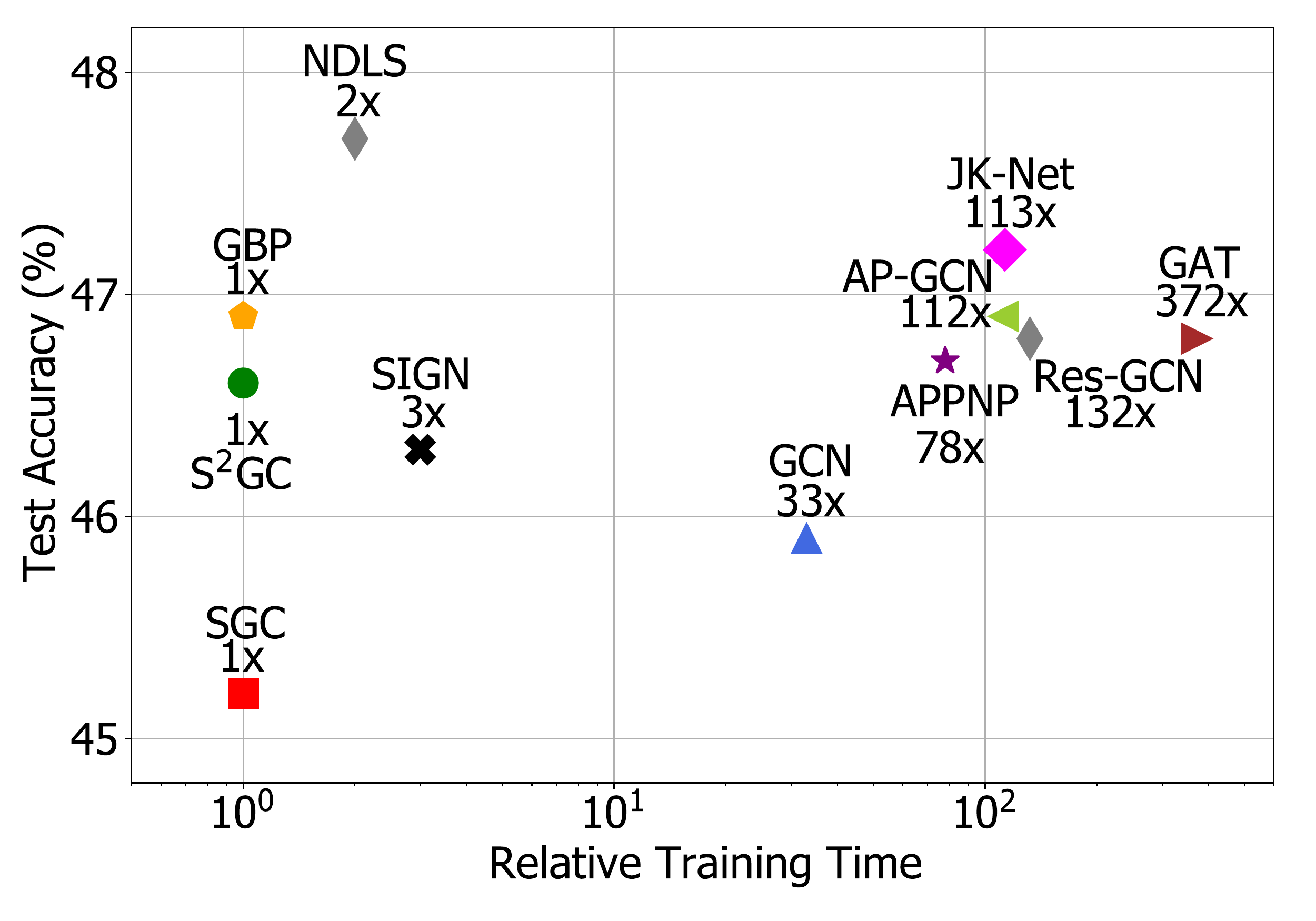}}
    \captionof{figure}{Performance along with training time on the Industry dataset.}
    \label{fig:efficiency-NDLS}
  \end{minipage}
  \hspace{4mm}
  \resizebox{0.45\textwidth}{!}{
  \begin{minipage}[b]{0.48\textwidth}
     \centering
     \captionof{table}{Results of inductive settings.} 
     \begin{tabular}{cccc}
        \toprule
        \textbf{Models}& \textbf{Flickr}&\textbf{Reddit}\\
        \midrule
        GraphSAGE & 50.1$\pm$1.3 & 95.4$\pm$0.0 \\
        FastGCN & 50.4$\pm$0.1 & 93.7$\pm$0.0 \\
        ClusterGCN & 48.1$\pm$0.5 & 95.7$\pm$0.0 \\
        GraphSAINT & 51.1$\pm$0.1 & 96.6$\pm$0.1\\
        \midrule
        NDLS-F+MLP$^*$& 51.9$\pm$0.2  & 96.6$\pm$0.1  \\
        GraphSAGE+NDLS-L$^*$& 51.5$\pm$0.4  & 96.3$\pm$0.0  \\
        NDLS$^*$&\textbf{52.6$\pm$0.4}&\textbf{96.8$\pm$0.1}\\
        \bottomrule
       \end{tabular}
    \label{inductive}
    \end{minipage}
    }
  \end{minipage}

\subsection{Experimental Results.}  

\textbf{End-to-end comparison.}  To answer \textbf{Q1}, Table~\ref{transductive} and ~\ref{inductive} show the test accuracy of considered methods in transductive and inductive settings. 
In the inductive setting, NDLS outperforms one of the most competitive baselines -- GraphSAINT by a margin of $1.5\%$ and $0.2\%$ on Flickr and Reddit. 
NDLS exceeds the best GNN model among all considered baselines on each dataset by a margin of $0.2\%$ to $0.8\%$ in the transductive setting.
In addition, we observe that with NDLS-L, the model performance of MLP, SGC, NDLS-F+MLP, and GraphSAGE can be further improved by a large margin. 
For example, the accuracy gain for MLP is $21.8\%$, $11.3\%$, $8.4\%$, and $5.6\%$ on Cora, Citseer, PubMed, and Industry, respectively.
To answer \textbf{Q2}, we replace the gate mechanism in the vanilla DAGNN with NDLS-L and refer to this method as DAGNN (NDLS-L). 
Surprisingly, DAGNN (NDLS-L) achieves at least comparable or (often) higher test accuracy compared with AP-GCN and DAGNN (Gate), and it shows that NDLS-L performs better than the learned mechanism in label smoothing.
Furthermore, by replacing the original graph kernels with NDLS-F, NDLS-F+MLP outperforms both S$^2$GC and GBP on all compared datasets. 
This demonstrates the effectiveness of the proposed NDLS.

\textbf{Training Efficiency.} To answer \textbf{Q3}, we evaluate the efficiency of each method on a real-world industry graph dataset.
Here, we pre-compute the smoothed features of each linear-model-based GNN, and the time for pre-processing is also included in the training time.
Figure~\ref{fig:efficiency-NDLS} illustrates the results on the industry dataset across training time.
Compared with linear-model-based GNNs, we observe that (1) both the coupled and decoupled GNNs require a significantly larger training time; (2) NDLS achieves the best test accuracy while consuming comparable training time with SGC.

\begin{figure*}[tp!]
\centering  
\subfigure[Feature Sparsity]{
\includegraphics[width=0.32\textwidth]{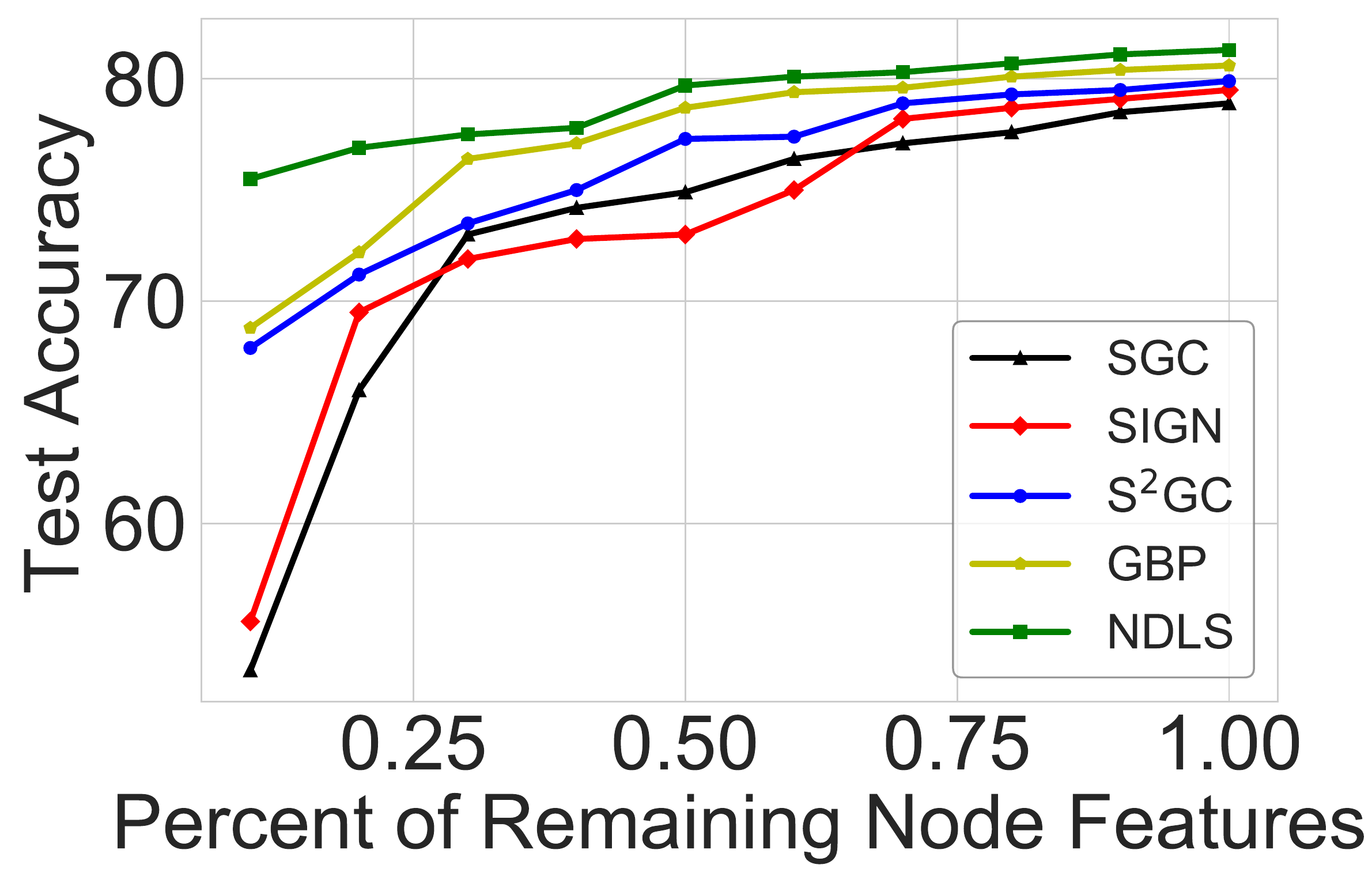}}\hspace{-1mm}
\subfigure[Edge Sparsity]{
\includegraphics[width=0.32\textwidth]{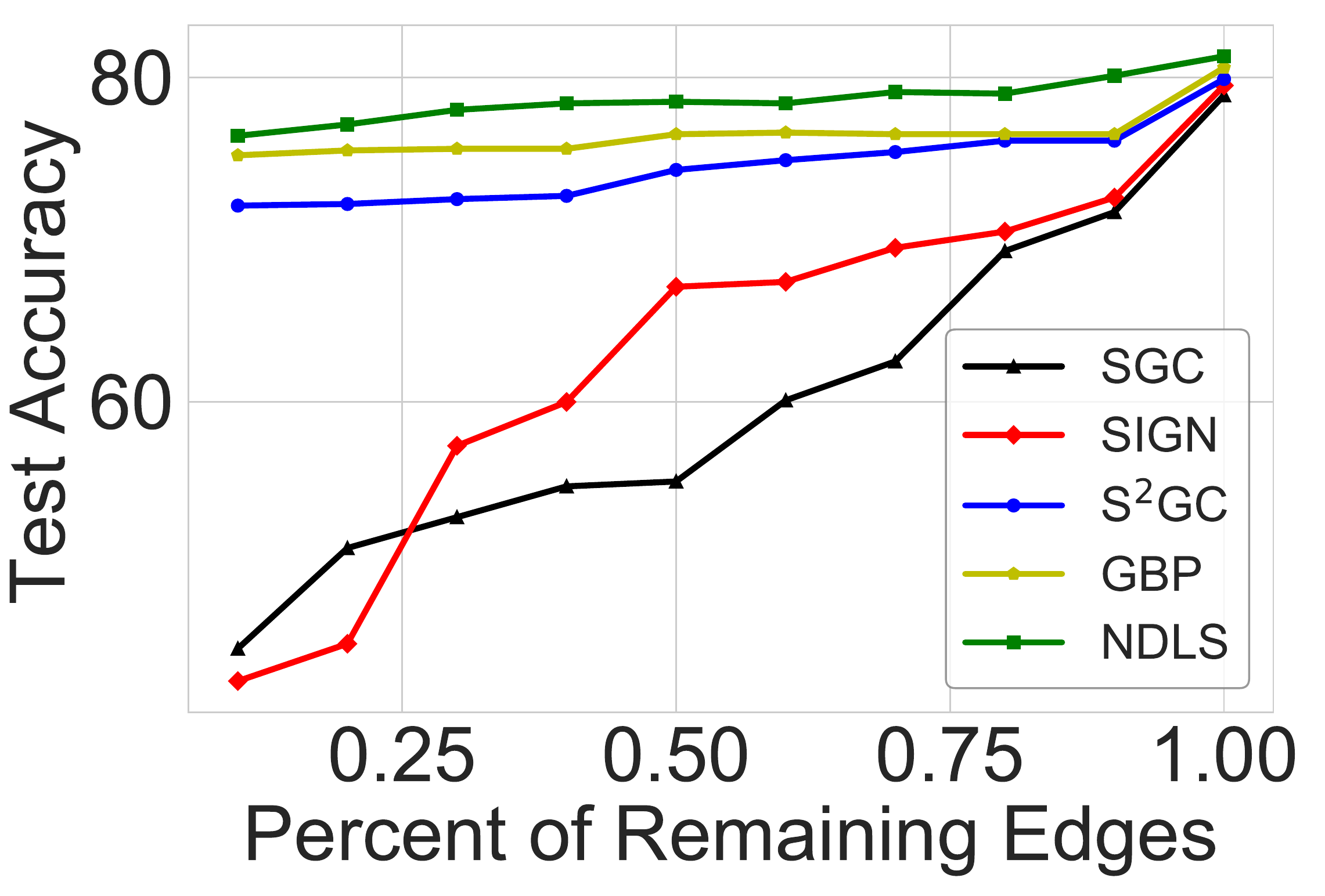}}\hspace{-1mm}
\subfigure[Label Sparsity]{
\includegraphics[width=0.32\textwidth]{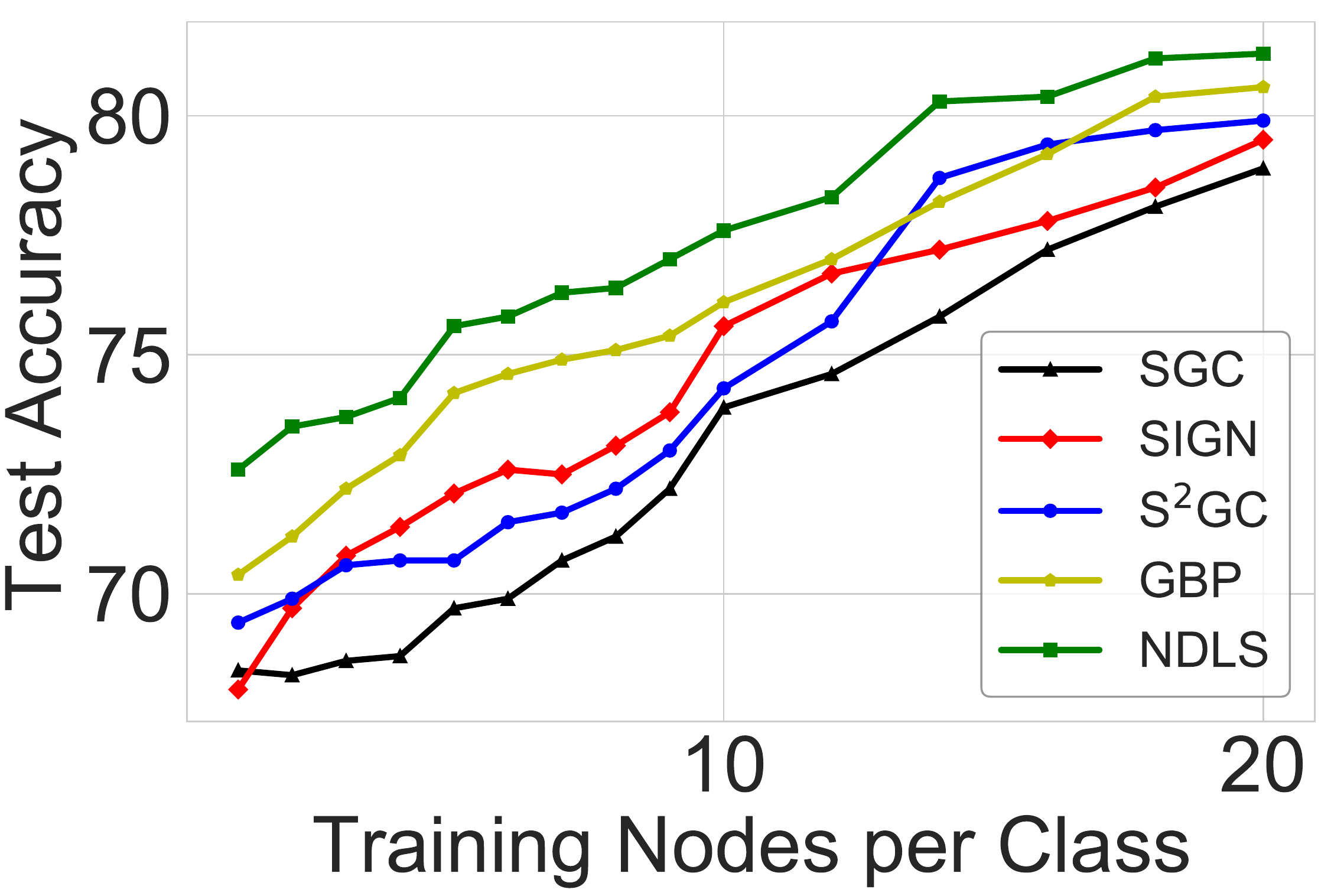}}\hspace{-1mm}
\caption{Test accuracy on PubMed dataset under different levels of feature, edge and label sparsity.}
\label{fig.sparsity}
\vspace{-0.8em}
\end{figure*}

\textbf{Performance on Sparse Graphs.}
To reply \textbf{Q4}, we conduct experiments to test the performance of NDLS on feature, edge, and label sparsity problems. 
For feature sparsity, we assume that the features of unlabeled nodes are partially missing. 
In this scenario, it is necessary to calculate a personalized propagation iteration to ``recover'' each node's feature representation. 
To simulate edge sparsity settings, we randomly remove a fixed percentage of edges from the original graph. 
Besides, we enumerate the number of nodes per class from 1 to 20 in the training set to measure the effectiveness of NDLS given different levels of label sparsity.
The results in Figure~\ref{fig.sparsity} show that NDLS outperforms all considered baselines by a large margin across different levels of feature, edge, and label sparsity, thus demonstrating that our method is more robust to the graph sparsity problem than the linear-model-based GNNs.

\begin{figure*}[tp!]
\centering  
\subfigure[LSI along with the node degree]{
\label{LSI-degree}
\scalebox{0.22}{
\includegraphics{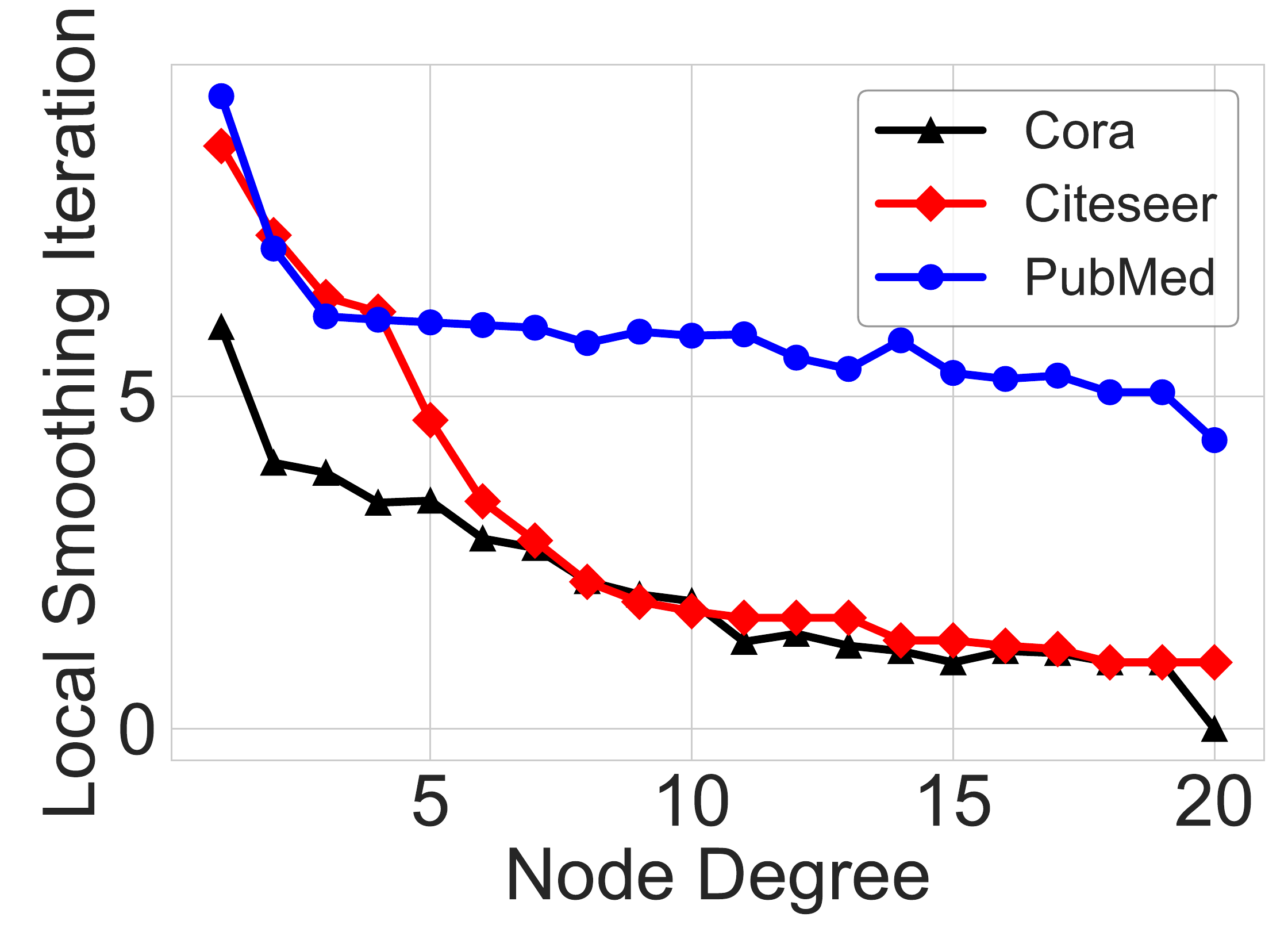}}}\hspace{8mm}
\centering
\subfigure[The visualization of LSI]{
\label{vis}
\scalebox{0.22}{
\includegraphics{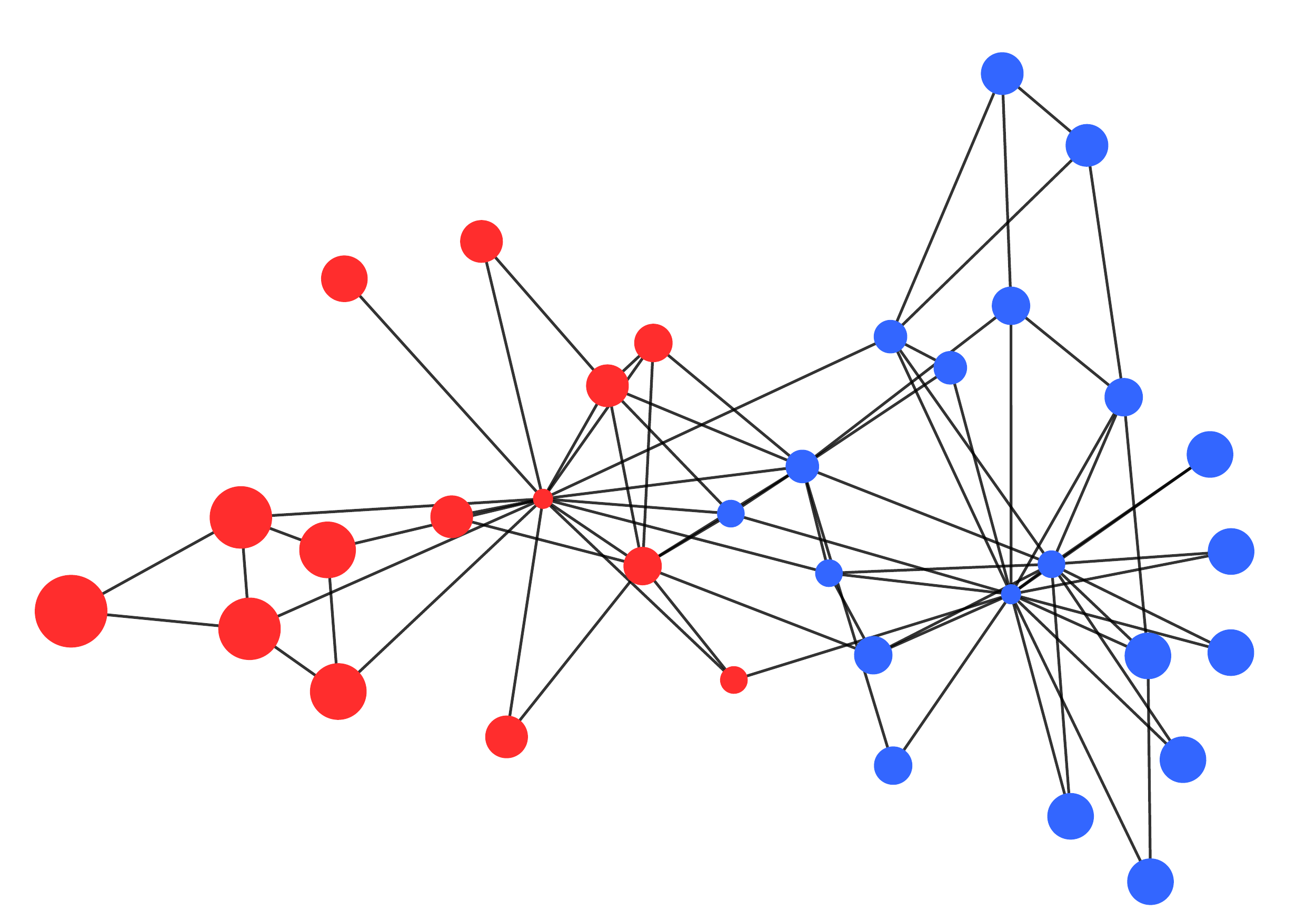}}}
\caption{(Left) LSI distribution along with the node degree in three citation networks. (Right) The visualization of LSI in Zachary's karate club network. Nodes with larger radius have larger LSIs.}
\label{interpretability}
\vspace{-0.8em}
\end{figure*}  

\textbf{Interpretability. }
As mentioned by \textbf{Q1}, we here answer why NDLS is effective.
One theoretical property of LSI is that the value correlates with the node degree negatively. 
We divide nodes into several groups, and each group consists of nodes with the same degree. And then we calculate the average LSI value for each group in the three citation networks respectively.
Figure ~\ref{LSI-degree} depicts that nodes with a higher degree have a smaller LSI, which is consistent with Theorem~\ref{theorem3.1}.
We also use NetworkX~\cite{hagberg2008exploring} to visualize the LSI in Zachary's karate club network~\cite{zachary1977information}. Figure~\ref{vis}, where the radius of each node corresponds to the value of LSI, shows three interesting observations: (1) nodes with a larger degree have smaller LSIs; (2) nodes in the neighbor area have similar LSIs; (3) nodes adjacent to a super-node have smaller LSIs.
The first observation is consistent with Theorem~\ref{theorem3.1}, and the latter two observations show consistency with Theorem~\ref{theorem3.2}.

\section{Conclusion}
In this paper, we present node-dependent local smoothing (NDLS), a simple and scalable graph learning method based on the local smoothing of features and labels. 
NDLS theoretically analyzes what influences the smoothness and gives a bound to guide how to control the extent of smoothness for different nodes.
By setting a node-specific smoothing iteration, each node in NDLS can smooth its feature/label to a local-smoothing state and then help to boost the model performance.
Extensive experiments on seven real-world graph datasets demonstrate the high accuracy, scalability, efficiency, and flexibility of NDLS against the state-of-the-art GNNs. 

\section*{Broader Impact}
NDLS can be employed in areas where graph modeling is the foremost choice, such as citation networks, social networks, chemical compounds, transaction graphs, road networks, etc. The effectiveness of NDLS when improving the predictive performance in those areas may bring a broad range of societal benefits. For example, accurately predicting the malicious accounts on transaction networks can help identify criminal behaviors such as stealing money and money laundering. Prediction on road networks can help avoid traffic overload and save people's time. A significant benefit of NDLS is that it offers a node-dependent solution. However, NDLS faces the risk of information leakage in the smoothed features or labels. In this regard, we encourage researchers to understand the privacy concerns of NDLS and investigate how to mitigate the possible information leakage.

\section*{Acknowledgments and Disclosure of Funding}
This work is supported by NSFC (No. 61832001, 6197200),  Beijing Academy of Artificial Intelligence (BAAI), PKU-Baidu Fund 2019BD006, and PKU-Tencent Joint Research Lab. Zhi Yang and Bin Cui are the corresponding authors.

\bibliographystyle{abbrv}
\bibliography{reference}

\appendix

\section{Appendix}

\subsection{Proofs of Theorems}

We represent the adjacency matrix and the diagonal degree matrix of graph $\mathcal{G}$ by $A$ and $D$ respectively, represent $D$+$I$ and $A$+$I$ by $\tilde{D}$ and $\tilde{A}$. Then we denote $\tilde{D}^{-1}\tilde{A}$ as a transition matrix $P$. Suppose $P$ is connected, which means the graph is connected, for any initial distribution $\pi_{0}$, let 
\begin{equation}
\tilde{\pi}(\pi_{0}) = \lim\limits_{k \to \infty }\pi_{0}P^k,
\end{equation}

then according to ~\cite{dihe2010introduction}, for any initial distribution $\pi_{0}$
\begin{equation}
\tilde{\pi}(\pi_{0})_i = \frac{1}{n}\sum\limits_{j=1}\limits^{n}P_{ji},
\end{equation}
where $\tilde{\pi}_i$ denotes the $i^{th}$ component of $\tilde{\pi}(\pi_{0})$,  and $n$ denotes the number of nodes in graph.
If matrix $P$ is unconnected, we can divide $P$ into connected blocks. Then for each blocks(denoted as $B_{g}$), there always be
\begin{equation}
\tilde{\pi}(\pi_{0})_i = \frac{1}{n_{g}}\sum\limits_{j\in B_{g}}P_{ji}*\sum\limits_{j\in B_{g}}\pi_{0j},
\end{equation}
where $n_{g}$ is the number of nodes in $B_{g}$. To make the proof concise, we will assume matrix $P$ is connected, otherwise we can perform the same operation inside each block. Therefore, $\tilde{\pi}$ is independent to $\pi_{0}$, thus we replace $\tilde{\pi}(\pi_{0})$ by $\tilde{\pi}$.

\begin{definition}[\textbf{Local Mixing Time}] The local mixing time (parameterized by $\epsilon$) with an initial distribution is defined as 
\begin{equation}
T(\pi_{0},\epsilon) = \min\{t:||\tilde{\pi}-\pi_{0}P^{t}||_{2}<\epsilon\},
\end{equation}
where "$||\cdot||_{2}$" symbols two-nor m.
\end{definition}

In order to consider the impact of each node to the others separately, let $\pi_{0}$ = $e_{i}$, where $e_{i}$ is a one-hot vector with the $i^{th}$ component equal to 1, and the other components equal to 0.
According to ~\cite{chung1997spectral} we have lemma \ref{lemma1}.
\begin{lemma}
\label{lemma1}
\begin{equation}
|(e_{i}P^{t})_{j}-\tilde{\pi}_j|\le \sqrt{\frac{\tilde{d_{j}}}{\tilde{d_{i}}}}\lambda_{2}^{t},
\end{equation}
where $\lambda_{2}$ is the second large eigenvalue of $P$ and $\tilde{d_{i}}$ denotes the degree of node $v_{i}$ plus 1 (to include itself).
\begin{equation*}
\tilde{d_{i}} = d_{i}+1,\quad \tilde{d_{j}} = d_{j}+1,
\end{equation*}
\end{lemma}
\begin{theorem}
\label{theorem1}
\begin{equation}
T(e_{i},\epsilon)\le \operatorname{log}_{\lambda_{2}}(\epsilon\sqrt{\frac{\tilde{d_{i}}}{2m+n}}),
\end{equation}
where $m$ and $n$ denote the number of edges and nodes in graph $\mathcal{G}$ separately.
\begin{equation*}
\tilde{d_{i}} = d_{i}+1,
\end{equation*}
\end{theorem}

\begin{proof}
~\cite{chung1997spectral} shows that when $\pi_{0}$ = $e_{i}$,
\begin{equation}
|(e_{i}P^{t})_{j}-\tilde{\pi}_j|\le \sqrt{\frac{\tilde{d_{j}}}{\tilde{d_{i}}}}\lambda_{2}^{t},
\end{equation}

where $(e_{i}P^{t})_{j}$ symbols the $j^{th}$ element of $e_{i}P^{t}$. We denote $e_{i}P^{t}$ as $\pi_{i}(t)$, then
\begin{equation}
\begin{split}
||\tilde{\pi}-\pi_{i}(t)||_{2}^{2}
&=\sum\limits_{j=1}\limits^{n}(\tilde{\pi}_{j}-\pi_{i}(t)_{j})^{2}\\
&\le\frac{\sum\limits_{j=1}\limits^{n}\tilde{d_{j}}}{\tilde{d_{i}}}\lambda_{2}^{2t}
=\frac{2m+n}{\tilde{d_{i}}}\lambda_{2}^{2t},
\end{split}
\end{equation}
which means 
\begin{equation}
||\tilde{\pi}-\pi_{i}(t)||_{2}\le\sqrt{\frac{2m+n}{\tilde{d_{i}}}}\lambda_{2}^{t}.
\end{equation}
Now let

\begin{equation*}
\epsilon = \sqrt{\frac{2m+n}{\tilde{d_{i}}}}\lambda_{2}^{t},
\end{equation*}

there exists
\begin{equation*}
T(e_{i},\epsilon)\le \operatorname{log}_{\lambda_{2}}(\epsilon\sqrt{\frac{\tilde{d_{i}}}{2m+n}}).
\end{equation*}
\end{proof}

Next consider the real situation in SGC with $n \times m$-dimension matrix $X(0)$ as input, where $n$ is the number of nodes, $m$ is the number of features.
We apply $P$ as the normalized adjacent matrix.(The definition of $P$ is the same as $\tilde{\textbf{A}}$ in main text). In feature propagation we have
\begin{equation*}
X(t) = P^{t}X(0),
\end{equation*}
Now consider the $h^{th}$ feature of $X$, we define an $n \times n$ influence matrix
\begin{equation}
I_{hij}(t) = \frac{\partial X(t)_{ih}}{\partial X(0)_{jh}},
\end{equation}
Because $I_{h}(k)$ is independent to $h$, we replace $I_{h}(k)$ by $I(k)$, which can be formulated as
\begin{equation}
I(k) = I_{h}(k), \quad \forall h \in \{1,2,..,f\},
\end{equation}
where $f$ symbols the number of features of $X$.
\begin{definition}[\textbf{Local Smoothing Iteration}] The Local Smoothing Iteration (parameterized by $\epsilon$) is defined as 
\begin{equation}
K(i,\epsilon) = \min\{k: ||\tilde{I}_{i}-I_{i}(k)||_{2}<\epsilon\}.
\end{equation}
\end{definition}

According to Theorem \ref{theorem1}, there exists 
\begin{theorem}[\textbf{Theorem 3.1 in main text}]
\label{theorem2}
When the normalized adjacent matrix is $P$, 
\begin{equation}
K(i,\epsilon)\le \operatorname{log}_{\lambda_{2}}(\epsilon\sqrt{\frac{\tilde{d_{i}}}{2m+n}}).
\end{equation}
\end{theorem}
\begin{proof}
From equation (9) we can derive that
\begin{equation*}
||e_{i}P^{\infty}-e_{i}P^k||_{2}\le\sqrt{\frac{2m+n}{\tilde{d_{i}}}}\lambda_{2}^{k}.
\end{equation*}
Because 
\begin{equation*}
I_{i}(k) = P^{k}_{i} = e_{i}P^{k} \quad I_{i}(\infty) = P^{\infty}_{i} = e_{i}P^{\infty},
\end{equation*}
we have
\begin{equation*}
||I_{i}(\infty)-I_{i}(k)||_{2}\le\sqrt{\frac{2m+n}{\tilde{d_{i}}}}\lambda_{2}^{k}.
\end{equation*}

Now let
\begin{equation*}
\epsilon = \sqrt{\frac{2m+n}{\tilde{d_{i}}}}\lambda_{2}^{k},
\end{equation*}
there exists
\begin{equation*}
K(i,\epsilon)\le \operatorname{log}_{\lambda_{2}}(\epsilon\sqrt{\frac{\tilde{d_{i}}}{2m+n}}).
\end{equation*}
\end{proof}
Therefore, we expand Theorem \ref{theorem2} to the propagation in SGC or our method. What is remarkable, Theorem \ref{theorem2} requires $P$, which is equal to $\tilde{D}^{-1}\tilde{A}$ as the normalized adjacent matrix. 

From Theorem \ref{theorem2} we can conclude that the node which has a lager degree may need more steps to propagate. At the same time, we have another bond of local mixing time as following.

\begin{theorem}
\label{theorem3}
For each node $v_{i}$ in graph $\mathcal{G}$, there always exits
\begin{equation}
T(e_{i},\epsilon)\le \max\{T(e_{j},\epsilon),j\in N(i)\}+1.
\end{equation}
where N(i) is the set of node $v_{i}$'s neighbours.
\end{theorem}

\begin{proof}
\begin{equation}
\begin{split}
\||\tilde{\pi}-e_{i}P^{t+1}||_{2}
&=\frac{1}{|N(i)|}\sum\limits_{j\in N(i)}||\tilde{\pi}-e_{j}P^{t}||_2\\
&\le \max\limits_{j\in N(i)}||\tilde{\pi}-e_{j}P^{t}||_2.
\end{split}
\end{equation}
Therefore, when 
\begin{equation*}
\max\limits_{j\in N(i)}||\tilde{\pi}-e_{j}P^{t}||_2\le\epsilon,
\end{equation*}
there exists 
\begin{equation*}
||\tilde{\pi}-e_{i}P^{t+1}||_{2}\le\epsilon.
\end{equation*}
Thus we can derive that 
\begin{equation*}
T(e_{i},\epsilon)\le \max\{T(e_{j},\epsilon),j\in N(i)\}+1.
\end{equation*}
\end{proof}

As we extend Theorem \ref{theorem1} to Theorem \ref{theorem2}, according to Theorem \ref{theorem3}, there always be
\begin{theorem}[\textbf{Theorem 3.2 in main text}]
For each node $v_{i}$ in graph $\mathcal{G}$, there always exits
\begin{equation}
K(i,\epsilon)\le \max\{K(j,\epsilon),j\in N(i)\}+1.
\end{equation}
\end{theorem}

\subsection{Results with More Base Models}

Our proposed NDLS consists of three stages: (1) feature smoothing with NDLS (NDLS-F), (2) model training with smoothed features, and (3) label smoothing with NDLS (NDLS-L).
In stage (2), the default option of the base model is a Multilayer Perceptron (MLP).
Besides MLP, many other models can also be used in stage (2) to generate soft labels.
To verify it, here we replace the MLP in stage (2) with popular machine learning models Random Forest~\cite{liaw2002classification} and XGBoost~\cite{chen2016xgboost}, and measure their node classification performance on PubMed dataset.
The experiment results are shown in Table~\ref{transductive} where Random Forest and XGBoost are abbreviated as {\em RF} and {\em XGB} respectively.

Compared to the vanilla model, both Random Forest and XGBoost achieve significant performance gain with the addition of our NDLS. 
With the help of NDLS, Random Forest and XGBoost outperforms their base models by $6.1\%$ and $7.5\%$ respectively.
From Table~\ref{transductive}, we can observe that both NDLS-F and NDLS-L can contribute great performance boost to the base model, where the gains are at least $5\%$.
When all equipped with both NDLS-F and NDLS-L, XGBoost beat the default MLP, achieving a test accuracy of $81.6\%$. 
Although Random Forest -- $80.5\%$ -- cannot outperform the other two models, it is still a competitive model. 

The above experiment demonstrates that the base model selection in stage (2) is rather flexible in our NDLS. Both traditional machine learning methods and neural networks are promising candidates in the proposed method.


\begin{table}[tpb!]
\caption{Results of different base models on PubMed.}
\vspace{-2mm}
\centering
{
\noindent
\renewcommand{\multirowsetup}{\centering}
\resizebox{0.7\linewidth}{!}{
\begin{tabular}{ccccccccccc}
\toprule
\textbf{Base Models}&\textbf{Models}&\textbf{Accuracy}& \textbf{Gain}\\
\midrule
\multirowcell{4}{MLP}
&Base &72.7$\pm$0.6&-\\
&+ NDLS-F& 81.1$\pm$0.6&+ 8.4\\ 
&+ NDLS-L&81.1$\pm$0.6&+ 8.4\\
&+ NDLS (both)&\textbf{81.4$\pm$0.4}&+ 8.7\\
\midrule
\multirowcell{4}{RF}
&Base &74.4$\pm$0.2&-\\
&+ NDLS-F&80.3$\pm$0.1&+ 5.9\\ 
&+ NDLS-L&80.0$\pm$0.2&+ 5.6\\
&+ NDLS (both)&\textbf{80.5$\pm$0.4}&+ 6.1\\
\midrule
\multirowcell{4}{XGB}
&Base &74.1$\pm$0.2&-\\
&+ NDLS-F &81.0$\pm$0.3&+ 6.9\\ 
&+ NDLS-L&79.8$\pm$0.2&+ 5.7\\
&+ NDLS (both)&\textbf{81.6$\pm$0.3}&+ 7.5\\
\bottomrule
\end{tabular}}}
\label{transductive2}
\end{table}

\subsection{Dataset Description}

\textbf{Cora}, \textbf{Citeseer}, and \textbf{Pubmed}\footnote{https://github.com/tkipf/gcn/tree/master/gcn/data} are three popular citation network datasets, and we follow the public training/validation/test split in GCN~\cite{kipf2016semi}.
In these three networks, papers from different topics are considered as nodes, and the edges are citations among the papers.  The node attributes are binary word vectors, and class labels are the topics papers belong to.

\noindent\textbf{Reddit} is a social network dataset derived from the community structure of numerous Reddit posts. It is a well-known inductive training dataset, and the training/validation/test split in our experiment is the same as the one in GraphSAGE~\cite{hamilton2017inductive}. 

\noindent\textbf{Flickr} originates from NUS-wide~\footnote{http://lms.comp.nus.edu.sg/research/NUS-WIDE.html} and contains different types of images based on the descriptions and common properties of online images. The public version of Reddit and Flickr provided by GraphSAINT\footnote{https://github.com/GraphSAINT/GraphSAINT} is used in our paper.

\noindent\textbf{Industry} is a short-form video graph, collected from a real-world mobile application from our industrial cooperative enterprise. 
 We sampled 1,000,000 users and videos from the app, and treat these items as nodes. The edges in the generated bipartite graph represent that the user clicks the short-form videos. 
 Each user has 64 features and the target is to category these short-form videos into 253 different classes. 
 
 \noindent\textbf{ogbn-papers100M} is a directed citation graph of 111 million papers indexed by MAG~\cite{wang2020microsoft}. Among its node set, approximately 1.5 million of them are arXiv papers, each of which is manually labeled with one of arXiv’s subject areas. Currently, this dataset is much larger than any existing public node classification datasets.
 
\begin{table}[tpb!]
\caption{URLs of baseline codes.}
\vspace{-2mm}
\centering
{
\noindent
\renewcommand{\multirowsetup}{\centering}
\resizebox{0.9\linewidth}{!}{
\begin{tabular}{cccc}
\toprule
\textbf{Type}&\textbf{Baselines}&\textbf{URLs}\\
\midrule
\multirowcell{2}{Coupled}&
GCN &https://github.com/rusty1s/pytorch\_geometric \\
&GAT &https://github.com/rusty1s/pytorch\_geometric \\
\midrule
\multirowcell{4}{Decoupled}
&APPNP &https://github.com/rusty1s/pytorch\_geometric \\
&PPRGo &https://github.com/TUM-DAML/pprgo\_pytorch\\
&AP-GCN & https://github.com/spindro/AP-GCN\\
&DAGNN & https://github.com/divelab/DeeperGNN\\
\midrule
\multirowcell{4}{Sampling}
&GraphSAGE & https://github.com/williamleif/GraphSAGE\\
&GraphSAINT &https://github.com/GraphSAINT/GraphSAINT\\
&FastGCN & https://github.com/matenure/FastGCN\\
&Cluster-GCN& https://github.com/benedekrozemberczki/ClusterGCN\\
\midrule
\multirowcell{5}{Linear}
&SGC &https://github.com/Tiiiger/SGC\\
&SIGN & https://github.com/twitter-research/sign\\
&S$^2$GC& https://github.com/allenhaozhu/SSGC\\
&GBP& https://github.com/chennnM/GBP\\
&NDLS&https://github.com/zwt233/NDLS\\
\bottomrule
\end{tabular}}}
\label{url}
\end{table} 
 
\subsection{Compared Baselines}
The main characteristic of all baselines are listed below: 
\begin{itemize}
    \item \textbf{GCN} ~\cite{kipf2016semi}: GCN is a novel and efficient method for semi-supervised classification on graph-structured data.
    \item \textbf{GAT} ~\cite{velivckovic2017graph}: GAT leverages masked self-attention layers to specify different weights to different nodes in a neighborhood, thus better represent graph information.
    \item \textbf{JK-Net} ~\cite{xu2018representation}: JK-Net is a flexible network embedding method that could gather different neighborhood ranges to enable better structure-aware representation. 
    \item \textbf{APPNP} ~\cite{klicpera2018predict}: APPNP uses the relationship between graph convolution networks (GCN) and PageRank to derive improved node representations.
    \item \textbf{AP-GCN} ~\cite{spinelli2020adaptive}: AP-GCN uses a halting unit to decide a receptive range of a given node. 
    \item \textbf{DAGNN} ~\cite{liu2020towards}: DAGNN proposes to decouple the representation transformation and propagation, and show that deep graph neural networks without this entanglement can leverage large receptive fields without suffering from performance deterioration.
    \item \textbf{PPRGo} ~\cite{bojchevski2020scaling}: utilizes an efficient approximation of information diffusion in GNNs resulting in significant speed gains while maintaining state-of-the-art prediction performance.
    \item \textbf{GraphSAGE} ~\cite{hamilton2017inductive}: GraphSAGE is an inductive framework that leverages node attribute information to efficiently generate representations on previously unseen data.
    \item \textbf{FastGCN} ~\cite{chen2018fastgcn}: FastGCN interprets graph convolutions as integral transforms of embedding functions under probability measures.
    \item \textbf{Cluster-GCN} ~\cite{chiang2019cluster}: Cluster-GCN is a novel GCN algorithm that is suitable for SGD-based training by exploiting the graph clustering structure.
    \item \textbf{GraphSAINT} ~\cite{DBLP:conf/iclr/ZengZSKP20}: GraphSAINT constructs mini-batches by sampling the training graph, rather than the nodes or edges across GCN layers.
    \item \textbf{SGC} ~\cite{wu2019simplifying}: SGC simplifies GCN by removing nonlinearities and collapsing weight matrices between consecutive layers. 
    \item \textbf{SIGN} ~\cite{rossi2020sign}: SIGN is an efficient and scalable graph embedding method that sidesteps graph sampling in GCN and uses different local graph operators to support different tasks. 
    \item \textbf{S$^2$GC} ~\cite{zhu2021simple}: S$^2$GC uses a modified Markov Diffusion Kernel to derive a variant of GCN, and it can be used as a trade-off of low-pass and high-pass filter which captures the global and local contexts of each node.
    \item \textbf{GBP} ~\cite{chen2020scalable}: GBP utilizes a localized bidirectional propagation process from both the feature vectors and the training/testing nodes
\end{itemize}
Table~\ref{url} summarizes the github URLs of the compared baselines. Following the original paper, we implement JK-Net by ourself since there is no official version available.

\subsection{Implementation Details}
\paragraph{Hyperparameter details.}
In stage (1), when computing the Local Smoothing Iteration, the maximal value of $k$ in equation (12) is set to 200 and the optimal $\epsilon$ value is get by means of a grid search from \{0.01, 0.03, 0.05\}.
In stage (2), we use a simple two-layer MLP to get the base prediction. The hidden size is set to 64 in small datasets -- Cora, Citeseer and Pubmed. While in larger datasets -- Flicker, Reddit, Industry and ogbn-papers100M, the hidden size is set to 256.
As for the dropout percentage and the learning rate, we use a grid search from \{0.2, 0.4, 0.6, 0.8\} and \{0.1, 0.01, 0.001\} respectively.
In stage (3), during the computation of the Local Smoothing Iteration, the maximal value of $k$ is set to 40. The optimal value of $\epsilon$ is obtained through the same process in stage (1).

\paragraph{Implementation environment.}
The experiments are conducted on a machine with Intel(R) Xeon(R) Gold 5120 CPU @ 2.20GHz, and a single NVIDIA TITAN RTX GPU with 24GB memory.
The operating system of the machine is Ubuntu 16.04.
As for software versions, we use Python 3.6, Pytorch 1.7.1 and CUDA 10.1.


\subsection{Comparison and Combination with Correct\&Smooth}
Similar to our NDLS-L, Correct and Smooth (C\&S) also applies post-processing on the model prediction. Therefore, we compare NDLS-L with C\&S below.

\begin{table}[tbp!]
\centering
{
\noindent
\renewcommand{\multirowsetup}{\centering}
\resizebox{0.8\linewidth}{!}{
\begin{tabular}{c|cccc}
\toprule
\textbf{Methods}&\textbf{Cora}& \textbf{Citeseer}&\textbf{PubMed}&\textbf{ogbn-papers100M}\\
\midrule
MLP+C\&S & 87.2 & 76.6 & 88.3 & 63.9 \\
MLP+NDLS-L & \textbf{88.1} & \textbf{78.3} & \textbf{88.5} & \textbf{64.6} \\
\bottomrule
\end{tabular}}}
\caption{Performance comparison between C\&S and NDLS-L} 
\label{vs_ndlsl}
\end{table}

\begin{table}[tbp!]
\centering
{
\noindent
\renewcommand{\multirowsetup}{\centering}
\resizebox{0.75\linewidth}{!}{
\begin{tabular}{c|cccccc}
\toprule
\textbf{Methods}&\textbf{2\%}& \textbf{5\%}&\textbf{10\%}&\textbf{20\%}&\textbf{40\%}&\textbf{60\%}\\
\midrule
MLP+S & \underline{63.1} & \underline{77.8} & 82.6 & 84.2 & 85.4 & 86.4 \\ 
MLP+C\&S & 62.8 & 76.7 & \underline{82.8} & \underline{84.9} & \underline{86.4} & \underline{87.2} \\
MLP+NDLS-L & \textbf{77.4} & \textbf{83.9} & \textbf{85.3} & \textbf{86.5} & \textbf{87.6} & \textbf{88.1} \\
\bottomrule
\end{tabular}}}
\caption{Performance comparison under varied label rate on the Cora dataset.} 
\label{diff_label_rate}
\end{table}

\textbf{Adaptivity to node.} C\&S adopts a propagation scheme based on Personalized PageRank (PPR), which always maintains certain input information to slow down the occurrence of over-smoothing. The expected number of smoothing iterations is controlled by the restart probability, which is a constant for all nodes. Therefore, C\&S still falls into the routine of fixed smoothing iteration. Instead, NDLS-L employs node-specific smoothing iterations. We compare each method's performance (test accuracy, \%) under the same data split as in the C\&S paper (60\%/20\%/20\% on three citation networks, official split on ogbn-papers100M), and the experimental results in Table~\ref{vs_ndlsl} show that NDLS-L outperforms C\&S in different datasets.

\textbf{Sensitivity to label rate.} During the ``Correct'' stage, C\&S propagates uncertainties from the training data across the graph to correct the base predictions. However, the uncertainties might not be accurate when the number of training nodes is relatively small, thus even degrading the performance. To confirm the above assumption, we conduct experiments on the Cora dataset under different label rates, and the experimental results are provided in Table~\ref{diff_label_rate}. As illustrated, the result of C\&S drops much faster than NDLS-L's when the label rate decreases. What's more, MLP+S (removing the ``Correct'' stage) outperforms MLP+C\&S when the label rate is low as expected.

Compared with C\&S, NDLS is more general in terms of smoothing types. C\&S can only smooth label predictions. Instead, NDLS can smooth both node features and label predictions and combine them to boost the model performance further.

\begin{table}[tbp!]
\centering
{
\noindent
\renewcommand{\multirowsetup}{\centering}
\resizebox{0.6\linewidth}{!}{
\begin{tabular}{c|ccc}
\toprule
\textbf{Methods}&\textbf{Cora}& \textbf{Citeseer}&\textbf{PubMed}\\
\midrule
MLP+C\&S & 76.7 & 70.8 & 76.5\\
MLP+C\&S+nd & \textbf{79.9} & \textbf{71.1} & \textbf{78.4}\\
\bottomrule
\end{tabular}}}
\caption{Performance comparison after combining the node-dependent idea with C\&S.}
\label{add_nd}
\end{table}

\textbf{Node Adaptive C\&S.} The node-dependent mechanism in our NDLS can easily be combined with C\&S. The two stages of C\&S both contain a smoothing process using the personalized PageRank matrix, where a coefficient controls the remaining percentage of the original node feature. Here, we can precompute the smoothed node features after the same smoothing step yet under different values like 0.1, 0.2, ..., 0.9. After that, we adopt the same strategy in our NDLS: for each node, we choose the first in the ascending order that the distance from the smoothed node feature to the stationarity is less than a tuned hyperparameter. By this means, the smoothing process in C\&S can be carried out in a node-dependent way.

We also evaluate the performance of C\&S combined with the node-dependent idea (represented as C\&S+nd) on the three citation networks under official splits, and the experimental results in Table~\ref{add_nd} show that C\&S combined with NDLS consistently outperforms the original version of C\&S.

\begin{table}[tbp!]
\centering
{
\noindent
\renewcommand{\multirowsetup}{\centering}
\resizebox{.95\linewidth}{!}{
\begin{tabular}{c|ccccccccccc}
\toprule
&\textbf{SGC}& \textbf{S$^2$GC}&\textbf{GBP}&\textbf{NDLS}&\textbf{SIGN}&\textbf{JK-Net}&\textbf{DAGNN}&\textbf{GCN}&\textbf{ResGCN}&\textbf{APPNP}&\textbf{GAT}\\
\midrule
Time & 1.00 & 1.19 & 1.20 & 1.50 & 1.59 & 11.42 & 14.39 & 20.43 & 20.49 & 28.88 & 33.23\\
Accuracy & 78.9 & 79.9 & 80.6 & \textbf{81.4} & 79.5 & 78.8 & 80.5 & 79.3 & 78.6 & 80.1 & 79.0\\
\bottomrule
\end{tabular}}}
\caption{Efficiency comparison on the PubMed dataset.}
\label{effi_pubmed}
\end{table}

\subsection{Training Efficiency Study}
we measure the training efficiency of the compared baselines on the widely used PubMed dataset. Using the training time of SGC as the baseline, the relative training time and the corresponding test accuracy of NDLS and the baseline methods are shown in Table~\ref{effi_pubmed}. Compared with other baselines, NDLS can get the highest test accuracy while maintaining competitive training efficiency.

\subsection{Experiments on ogbn-arxiv}
We also conduct experiments on the ogbn-arxiv dataset. The experiment results (test accuracy, \%) are provided in Table~\ref{perf_arxiv}. Although GAT outperforms NDLS on ogbn-arxiv dataset, it is hard to scale to large graphs like ogbn-papers100M dataset. Note that MLP+C\&S on the OGB leaderboard makes use of not only the original node feature but also diffusion embeddings and spectral embeddings. Here we remove the latter two embeddings for fairness, and the authentic MLP+C\&S achieves 71.58\% on the ogbn-arxiv dataset.

\begin{table}[tbp!]
\centering
{
\noindent
\renewcommand{\multirowsetup}{\centering}
\resizebox{.95\linewidth}{!}{
\begin{tabular}{c|ccccccccccc}
\toprule
&\textbf{MLP}& \textbf{MLP+C\&S}&\textbf{GCN}&\textbf{SGC}&\textbf{SIGN}&\textbf{DAGNN}&\textbf{JK-Net}&\textbf{S$^2$GC}&\textbf{GBP}&\textbf{NDLS}&\textbf{GAT}\\
\midrule
Accuracy & 55.50 & 71.58 & 71.74 & 71.72 & 71.95 & 72.09 & 72.19 & 72.21 & 72.45 & \underline{73.04} & \textbf{73.56} \\
\bottomrule
\end{tabular}}}
\caption{Performance comparison on the ogbn-arxiv dataset.}
\label{perf_arxiv}
\end{table}

\end{document}